\documentclass{article}

\usepackage{microtype}
\usepackage{graphicx}
\usepackage{subcaption}
\usepackage{booktabs}

\usepackage{hyperref}
\usepackage{multirow}

\usepackage[accepted]{icml2026}

\usepackage{amsmath}
\usepackage{amssymb}
\usepackage{mathtools}
\usepackage{amsthm}

\usepackage{pifont}

\usepackage[capitalize,noabbrev]{cleveref}

\usepackage{fontawesome5}

\theoremstyle{plain}
\newtheorem{theorem}{Theorem}[section]
\newtheorem{proposition}[theorem]{Proposition}

\theoremstyle{definition}

\newtheorem{assumption}[theorem]{Assumption}
\theoremstyle{remark}
\newtheorem{remark}[theorem]{Remark}

\icmltitlerunning{Dual-Coupled PnP Diffusion (DC-PnPDP)}

\begin{document}

\twocolumn[
    \icmltitle{Plug-and-Play Diffusion Meets ADMM: Dual-Variable Coupling for Robust Medical Image Reconstruction}

    \icmlsetsymbol{equal}{*}
    \icmlsetsymbol{corr}{\faIcon{envelope}}

    \begin{icmlauthorlist}
        \icmlauthor{Chenhe Du}{shanghaitech}
        \icmlauthor{Xuanyu Tian}{shanghaitech}
        \icmlauthor{Qing Wu}{shanghaitech}
        \icmlauthor{Muyu Liu}{shanghaitech}
        \icmlauthor{Jingyi Yu}{shanghaitech}
        \icmlauthor{Hongjiang Wei}{sjtu}
        \icmlauthor{Yuyao Zhang}{shanghaitech,corr}
    \end{icmlauthorlist}

    \icmlaffiliation{shanghaitech}{ShanghaiTech University; }
    \icmlaffiliation{sjtu}{Shanghai Jiao Tong University}

    \icmlcorrespondingauthor{Yuyao Zhang}{zhangyy8@shanghaitech.edu.cn}

    \icmlkeywords{Machine Learning, ICML}

    \vskip 0.3in
]

\printAffiliationsAndNotice{}

\begin{abstract}
    Plug-and-Play diffusion prior (PnPDP) frameworks have emerged as a powerful paradigm for solving imaging inverse problems by treating pretrained generative models as modular priors. However, we identify a critical flaw in prevailing PnP solvers (e.g., based on HQS or Proximal Gradient): they function as \textit{memoryless operators}, updating estimates solely based on instantaneous gradients. This lack of historical tracking inevitably leads to non-vanishing \textit{steady-state bias}, where the reconstruction fails to strictly satisfy physical measurements under heavy corruption. To resolve this, we propose \textbf{Dual-Coupled PnP Diffusion (DC-PnPDP)}, which restores the classical \textit{dual variable} to provide integral feedback, progressively enforce agreement between the data-consistency and prior. However, this rigorous geometric coupling introduces a secondary challenge: the accumulated dual residuals exhibit spectrally colored, structured artifacts that violate the Additive White Gaussian Noise (AWGN) assumption of diffusion priors, causing severe hallucinations. To bridge this gap, we introduce \textbf{Spectral Homogenization (SH)}, a frequency-domain adaptation mechanism that modulates these structured residuals into statistically compliant \textit{pseudo-AWGN} inputs. This effectively aligns the solver's rigorous optimization trajectory with the denoiser's valid statistical manifold. Extensive experiments on CT and MRI reconstruction demonstrate that our approach resolves the bias-hallucination trade-off, achieving state-of-the-art fidelity with significantly accelerated convergence. The code is available at \url{https://github.com/duchenhe/DC-PnPDP}.
\end{abstract}

\section{Introduction}
\label{sec:intro}

Medical image reconstruction sits at the core of modern diagnostic workflows, bridging the gap between raw sensor data and clinical interpretation. In modalities such as Computed Tomography (CT) and Magnetic Resonance Imaging (MRI), physical constraints—ranging from reduced radiation dose to accelerated acquisition times—inevitably render the reconstruction task an ill-posed inverse problem. Mathematically, this is modeled as recovering a high-fidelity image $\mathbf{x} \in \mathbb{R}^n$ from degraded measurements $\mathbf{y}$:
\begin{equation}
    \mathbf{y} = \mathbf{A}\mathbf{x} + \mathbf{n},
    \label{eq:inverse_prob}
\end{equation}
where $\mathbf{A}$ denotes the forward operator (e.g., Radon transform or Fourier encoding) and $\mathbf{n}$ represents measurement noise. To resolve the ill-posedness, one must introduce prior knowledge to restrict the solution space to valid anatomical manifolds. Modern approaches increasingly leverage Diffusion Models (DMs) \cite{ho2020denoising, song2020score} as learned priors. The prevailing paradigm, known as Plug-and-Play Diffusion Prior (PnPDP) \cite{zheng2025inversebench}, decomposes the intractable posterior inference into two alternating sub-problems: a \textit{data-consistency step} that enforces fidelity to $\mathbf{y}$, and a \textit{prior step} that projects the estimate onto the data manifold using a pretrained diffusion model. This modularity has spurred a flurry of research: some works focus on designing more expressive diffusion backbones (enhancing the prior)~\cite{karras2022elucidating,dhariwal2021diffusion,rombach2022high,peebles2023scalable}, while others develop specialized solvers for the measurement operator (refining the likelihood)~\cite{chung2022diffusion,wang2022zero,song2021solving,chung2022improving,chung2024decomposed,zhu2023denoising,yang2023dsg,song2024solving,du2024dper,zhang2025improving}.

Despite the empirical success of PnP diffusion, we identify a subtle yet critical flaw in current solvers (e.g., those based on Half-Quadratic Splitting (HQS) or Proximal Gradient methods): they are inherently \textit{memoryless}. In each iteration, these algorithms update the solution based solely on the \textit{instantaneous} gradient of the data-fidelity term. From a control theory perspective, they function as Proportional (P) controllers \cite{bequette2003process}. It is a well-established theoretical result that P-controllers cannot eliminate steady-state bias when the system faces significant resistance (e.g., heavy measurement corruption or ill-conditioning). Consequently, current PnPDP solvers often converge to a biased equilibrium where the reconstruction fails to strictly satisfy the measurement constraints, effectively ``compromising" between the physics and the prior. This limitation is particularly pernicious in medical imaging: unlike natural image synthesis where perceptual plausibility is paramount, diagnostic reconstruction demands rigorous fidelity to the observed measurements $\mathbf{y}$. A solver that settles for a ``visually pleasing" approximation while violating physical constraints inherently compromises clinical reliability.

To eliminate this bias, classical optimization theory offers a rigorous solution: the \textit{Dual Variable} (or Lagrange Multiplier). As exemplified in the Alternating Direction Method of Multipliers (ADMM) \cite{boyd2011distributed}, the dual variable acts as an integral memory, accumulating historical constraint violations over time to generate a corrective force. This mechanism provides an integral feedback effect by accumulating historical consensus errors, thereby progressively driving the data-consistency and prior variables toward agreement under standard ADMM assumptions.

However, simply re-injecting this classical mechanism into modern diffusion loops proves practically non-trivial. We uncover a fundamental distributional conflict: the dual variable, by design, accumulates \textit{structured} errors (e.g., directional streaks in CT or coherent aliasing in MRI) to correct the optimization trajectory. When these structured, spectrally colored residuals are fed back into the diffusion model (which is strictly trained on \textit{Additive White Gaussian Noise}, AWGN), they constitute Out-of-Distribution (OOD) inputs. This mismatch causes the diffusion prior to misinterpret optimization artifacts as semantic content, leading to severe hallucinations or reconstruction divergence.
This creates a fundamental geometric-statistical conflict: the dual variable is essential for the integral action in optimization, ensuring convergence, yet its structured spectral bias violates the i.i.d. Gaussian assumption of pretrained diffusion priors, leading to severe hallucination.

In this work, we propose a principled framework to resolve this conflict, termed \textbf{Dual-Coupled PnP Diffusion Prior(DC-PnPDP)}. Our method bridges the gap between the rigorous geometry of ADMM and the statistical requirements of diffusion models. Instead of discarding the dual variable (as in current SOTA methods) or naively feeding it to the network (which causes OOD risks), we introduce a \textit{Spectral Homogenization (SH)} mechanism. This module operates in the frequency domain to modulate the structured ADMM residuals, complementing them with synthetic noise to construct \textit{Pseudo-AWGN} inputs on the fly. This allows the diffusion prior to operate within its valid statistical range while simultaneously preserving the dual-driven force essential for solving the optimization problem without bias. Our specific contributions are:
\begin{itemize}
    \item \textbf{Restoring the Dual Mechanism:} We re-introduce the explicit ADMM dual variable into the PnP diffusion loop, demonstrating that this ``memory'' term is essential for eliminating steady-state bias and ensuring robustness against heavy corruption.
    \item \textbf{Spectral Homogenization:} We propose a lightweight frequency-domain projection that transforms structured optimization residuals into spectrally whitened inputs. This effectively aligns the solver's trajectory with the denoiser's training distribution.
    \item \textbf{Efficiency and Performance:} Extensive experiments on limited-view CT and Accelerated MRI demonstrate that DC-PnPDP achieves superior inference speed (approx. $3\times$ faster than baselines) and reconstruction fidelity, establishing a new state-of-the-art (SOTA) for diffusion-based inverse solvers.
\end{itemize}

\section{Related Work}
\label{app:related_work}

\paragraph{Plug-and-play ADMM with learned denoisers.}
The idea of combining physical forward models with external denoisers was first formalized in the plug-and-play (PnP) framework, where the proximal map associated with an explicit regularizer is replaced by an off-the-shelf denoising operator~\citep{venkatakrishnan2013plug}. ADMM is a natural backbone for this formulation because variable splitting separates the data-fidelity update from the prior update while using a dual variable to enforce consensus. Early PnP-ADMM work demonstrated strong empirical performance and established fixed-point convergence under bounded-denoiser conditions~\citep{chan2016plug}. Subsequent theory refined the required denoiser assumptions: \citet{ryu2019plug} proved convergence of PnP-FBS and PnP-ADMM for properly trained Lipschitz denoisers; \citet{hurault2022proximal} constructed proximal denoisers that correspond to explicit nonconvex regularizers; and \citet{park2023convergence} related the empirical stability of PnP-ADMM with expansive CNN denoisers to the MMSE interpretation of denoising. These works establish that ADMM can be a principled interface between measurement models and learned priors, but they mainly concern discriminative denoisers rather than diffusion models trained over a continuum of noise levels.

\paragraph{Robustness, adaptation, and stochastic variants.}
Several works study how PnP-ADMM behaves when the denoiser, data distribution, or computational budget departs from the ideal assumptions. \citet{shoushtari2024prior} analyze prior mismatch in PnP-ADMM and show that domain adaptation can reduce the resulting performance gap under nonconvex data fidelity and expansive denoisers. In MRI reconstruction, \citet{hou2022truncated} propose a truncated-residual PnP-ADMM algorithm that preserves fixed-point convergence while using a neural denoiser with continuous noise conditioning. Stochastic variants pursue complementary goals: \citet{tang2020fast} introduce a stochastic PnP-ADMM scheme with inexact ADMM inner loops to reduce repeated denoiser evaluations, while \citet{coeurdoux2024plug} develop a split Gibbs sampler inspired by HQS and ADMM to obtain Bayesian posterior samples with deep generative priors. These methods highlight the flexibility of ADMM-style splitting, but they do not address the specific statistical mismatch that arises when the ADMM dual-shifted iterate is directly passed to a diffusion denoiser trained under an AWGN corruption model.

\paragraph{Diffusion priors inside PnP and ADMM-like solvers.}
Diffusion models have recently become powerful priors for inverse problems, either by modifying the reverse diffusion dynamics~\citep{chung2022diffusion, song2021solving, wang2022zero, zhang2025improving} or by using diffusion denoisers in PnP-style iterations~\citep{zhu2023denoising, zheng2025inversebench}. Very recent and largely concurrent works have begun to revisit the connection between diffusion priors, variable splitting, and ADMM-like optimization. \citet{park2026stochastic} provide a score-based interpretation of PnP and propose stochastic generative PnP, where injected noise allows pretrained score-based diffusion models to act as stronger PnP priors; however, their formulation does not explicitly study the ADMM dual variable or the out-of-distribution (OOD) denoiser input induced by dual-shifted iterates. \citet{shrestha2026taming} are particularly related to our motivation: they also identify a mismatch between score-model training distributions and ADMM iterates, and propose an AC-DC denoiser with auto-correction, directional correction, and score-based denoising. Their solution relies on conditional Langevin dynamics inside the denoising interface, which can require multiple inner sampling steps and therefore increases computational cost. By contrast, our Spectral Homogenization module performs a lightweight frequency-domain correction and does not increase the number of denoiser evaluations. \citet{zhang2025decoupling} use ADMM to decouple unconditional diffusion generation and differentiable guidance functions for training-free conditional generation, while \citet{kim2025dual} formulate inverse-problem MAP estimation with diffusion priors through dual ascent. In a related but non-ADMM direction, \citet{bendel2025solving} propose iterative colored re-noising to keep the pretrained diffusion model aligned with the AWGN input distribution.

\paragraph{Position of this work.}
Our work is closest in spirit to this very recent group of methods that use ADMM, dual ascent, or noise re-injection to strengthen the interaction between diffusion priors and measurement constraints. The key distinction is that we explicitly instantiate the ADMM dual variable inside a PnP diffusion reconstruction loop and study the resulting interface problem in medical imaging. This design restores the accumulated consensus feedback that is absent from memoryless HQS or proximal-gradient PnP diffusion solvers, while preserving compatibility with pretrained diffusion priors and efficient physics-based data-consistency updates. At the same time, directly injecting the dual variable creates spectrally structured residuals that are not AWGN-like. The proposed Spectral Homogenization module is therefore complementary to the dual-coupled update: it adapts the ADMM-induced residual distribution to the denoiser's training manifold without retraining the diffusion model, introducing Langevin inner loops, or increasing the number of denoiser evaluations.

\section{Background}
\label{sec:background}

\subsection{Inverse Problems and Variational Formulation}
We consider the generalized inverse problem of recovering an image $\mathbf{x} \in \mathbb{R}^n$ from ill-posed measurements $\mathbf{y} \in \mathbb{R}^m$, governed by the linear model $\mathbf{y} = \mathbf{A}\mathbf{x} + \mathbf{n}$. A standard resolution strategy lies in the variational framework, which estimates $\mathbf{x}$ by minimizing a composite objective:
\begin{equation}
    \min_{\mathbf{x}} \quad f(\mathbf{x}; \mathbf{y}) + \lambda \mathcal{R}(\mathbf{x}),
    \label{eq:variational_obj}
\end{equation}
where $f(\mathbf{x}; \mathbf{y}) = \frac{1}{2}\|\mathbf{A}\mathbf{x} - \mathbf{y}\|_\mathbf{W}^2$ enforces data consistency (weighted by noise precision $\mathbf{W}$), and $\mathcal{R}(\mathbf{x})$ is a regularization term encoding prior knowledge of the prior manifold.

\subsection{ADMM and the Dual Variable Mechanism}
To solve \eqref{eq:variational_obj}, the ADMM introduces an auxiliary variable $\mathbf{z}$ to decouple the fidelity and prior terms:
\begin{equation}
    \min_{\mathbf{x}, \mathbf{z}} \quad f(\mathbf{x}) + \lambda \mathcal{R}(\mathbf{z}) \quad \text{s.t.} \quad \mathbf{x} = \mathbf{z}.
\end{equation}
By forming the augmented Lagrangian, ADMM decomposes the problem into three iterative steps using a scaled dual variable $\mathbf{u}$:
\begin{align}
    \mathbf{x}^{(k+1)} & = \arg\min_{\mathbf{x}} f(\mathbf{x}) + \frac{\rho}{2}\|\mathbf{x} - \mathbf{z}^{(k)} + \mathbf{u}^{(k)}\|_2^2, \label{eq:admm_x} \\
    \mathbf{z}^{(k+1)} & = \text{prox}_{\frac{\lambda}{\rho}\mathcal{R}}\left(\mathbf{x}^{(k+1)} + \mathbf{u}^{(k)}\right), \label{eq:admm_z}              \\
    \mathbf{u}^{(k+1)} & = \mathbf{u}^{(k)} + \left(\mathbf{x}^{(k+1)} - \mathbf{z}^{(k+1)}\right). \label{eq:admm_u}
\end{align}

Here, $\rho > 0$ is a penalty parameter. Crucially, Eq.~\eqref{eq:admm_u} reveals that $\mathbf{u}^{(k)}$ acts as an integrator of the primal residual $(\mathbf{x} - \mathbf{z})$.
In control theory terms, this provides \textit{integral action}: it accumulates historical consensus errors to drive the system toward a theoretically rigorous fixed point where $\mathbf{x}^* = \mathbf{z}^*$ and measurement constraints are strictly satisfied, eliminating steady-state bias.

\subsection{Plug-and-Play Diffusion and the Statistical Gap}
The core insight of plug-and-play (PnP) priors \citep{venkatakrishnan2013plug,chan2016plug,kamilov2023plug} is that the proximal operator in Eq.\eqref{eq:admm_z}, formally defined as $\text{prox}_{\tau \mathcal{R}}(\mathbf{v}) = \arg\min_{\mathbf{z}} \frac{1}{2}\|\mathbf{z} - \mathbf{v}\|_2^2 + \tau \mathcal{R}(\mathbf{z})$, is mathematically equivalent to a maximum a posteriori (MAP) denoiser for AWGN. Consequently, PnP priors replace the explicit proximal operator in Eq.~\eqref{eq:admm_z} with an off-the-shelf denoiser $\mathcal{D}_{\sigma}$.
Modern PnP frameworks leverage Diffusion Models, where the denoiser $\mathcal{D}_{\sigma}(\cdot)$ is formally linked to the score function of the data distribution $p_{data}(\mathbf{x})$ via Tweedie's formula \cite{robbins1992empirical, song2020score}:
\begin{equation}
    \mathcal{D}_{\sigma}(\mathbf{v}) \approx \mathbf{v} + \sigma^2 \nabla_{\mathbf{v}} \log p_\sigma(\mathbf{v}) = \mathbb{E}[\mathbf{x}|\mathbf{v}].
    \label{eq:tweedie}
\end{equation}
This substitution implies that the PnP iteration effectively performs Gradient Descent on the data fidelity term and Score Ascent on the prior.

\paragraph{The Memoryless Approximation.} Most existing PnP diffusion solvers \citep{chung2022diffusion, zhu2023denoising,wu2024principled,wang2024dmplug,xu2024provably,chu2025split} typically adopt schemes akin to HQS, which effectively set $u^{(k)} \equiv 0$. While this reduces computational complexity, it discards the integral action of the dual variable, reducing the solver to a memoryless operator susceptible to asymptotic bias (as discussed in Sec.~\ref{sec:intro}).

\paragraph{The Distributional Conflict.} Reintroducing $\mathbf{u}^{(k)}$ is non-trivial due to the statistical constraints of the denoiser. Pretrained diffusion denoisers $\mathcal{D}_\sigma$ rely on a strict AWGN Assumption: they provide a valid MAP estimate \textit{if and only if} the input follows:
\begin{equation}
    \mathbf{v} = \mathbf{x}_\text{clean} + \mathbf{n}, \quad \text{where} \quad \mathbf{n} \sim \mathcal{N}(\mathbf{0}, \sigma^2 \mathbf{I}).
\end{equation}
However, the ADMM dual-shifted input $\mathbf{v}^{(k+1)} = \mathbf{x}^{(k+1)} + \mathbf{u}^{(k)}$ inherently contains structured, spatially correlated residuals (accumulated in $\mathbf{u}$). This spectral coloring violates the AWGN assumption, pushing the input OOD and causing the DMs to hallucinate artifacts as features.

\section{Method: Dual-Coupled PnP Diffusion}
\label{sec:method}

We present \textbf{Dual-Coupled PnP Diffusion (DC-PnPDP)}, a framework designed to reconcile the rigorous convergence requirements of variational optimization with the statistical rigidity of pretrained generative priors.
Our core insight is to decouple the \textit{geometric} and \textit{statistical} roles of the solver:
(1) We explicitly restore the ADMM dual variable to enforce asymptotic consensus (eliminating bias);
(2) To prevent the resulting structured dual residuals from disrupting the denoiser, we introduce \textbf{Spectral Homogenization (SH)}, a frequency-domain adaptation module. SH transforms the colored optimization artifacts into statistically compliant pseudo-AWGN, ensuring the diffusion prior operates strictly within its valid training manifold.

\subsection{The Dual-Coupled Iteration Scheme}
\label{sec:dual_coupling}
We adopt the variable splitting formulation from Eq.~\eqref{eq:admm_x}-\eqref{eq:admm_u}. Unlike memoryless baselines that discard the dual history, we strictly maintain the dual state $\mathbf{u}^{(k)}$ to track cumulative constraint violations. The overview iteration process, detailed in Algorithm \ref{alg:main}, proceeds in two coupled phases:

\paragraph{(1). Physics-Driven Update (Alg. \ref{alg:main}, Step 1).}
First, we solve the data-consistency sub-problem. Since $f(\mathbf{x})$ is quadratic, this admits a closed-form solution or can be efficiently solved via Conjugate Gradient (CG):
\begin{equation}
    \mathbf{x}^{(k+1)} = \arg\min_{\mathbf{x}} \|\mathbf{A}\mathbf{x} - \mathbf{y}\|_2^2 + \rho\|\mathbf{x} - (\mathbf{z}^{(k)} - \mathbf{u}^{(k)})\|_2^2.
\end{equation}
\paragraph{(2). The Dual-Shifted Interface (Alg. \ref{alg:main}, Steps 2--4).}
To update the prior estimate $\mathbf{z}$, standard ADMM would compute the proximal operator on the dual-shifted variable $\mathbf{v}^{(k+1)} \coloneqq \mathbf{x}^{(k+1)} + \mathbf{u}^{(k)}$.
As analyzed in Sec.~\ref{sec:background}, $\mathbf{v}^{(k+1)}$ is contaminated by spatially correlated, spectrally colored artifacts accumulated in $\mathbf{u}^{(k)}$ (e.g., aliasing patterns). Direct application of the denoiser $\mathcal{D}_\sigma(\mathbf{v}^{(k+1)})$ violates the AWGN assumption, leading to artifact hallucination.
Therefore, we insert our distribution-correction module $\mathcal{T}_{\text{SH}}$ before the denoiser:
\begin{equation}
    \mathbf{z}^{(k+1)} = \mathcal{D}_\sigma\left( \mathcal{T}_{\text{SH}}(\mathbf{v}^{(k+1)}; \mathbf{z}^{(k)}) \right).
\end{equation}
Finally, the dual variable is updated via $\mathbf{u}^{(k+1)} = \mathbf{u}^{(k)} + (\mathbf{x}^{(k+1)} - \mathbf{z}^{(k+1)})$, closing the feedback loop.
Crucially, the stability of this rigorous feedback mechanism relies entirely on the denoiser accepting the dual-shifted input without exhibiting OOD behavior. In the following section, we detail how Spectral Homogenization $\mathcal{T}_{\text{SH}}$ bridges this statistical gap by modulating the input's frequency characteristics.

\subsection{Spectral Homogenization (SH)}
\label{sec:sh}

The goal of Spectral Homogenization is to construct a rectified input $\tilde{\mathbf{v}}$ whose effective noise distribution matches the isotropic Gaussian statistics $\mathcal{N}(\mathbf{0}, \sigma_t^2 \mathbf{I})$ expected by the diffusion model, without destroying the low-frequency semantic content of $\mathbf{x}$.

\paragraph{Motivation: Why Frequency Domain Modulation?}
The rationale for intervening in the frequency domain, rather than the spatial domain, stems from the need to balance \textit{statistical whitening} with \textit{structure preservation}.
Physical artifacts (e.g., undersampling streaks or aliasing) are inherently \textit{spectrally colored}, manifesting as energy concentrations in specific frequency bands, whereas the diffusion model expects a flat (white) spectrum. A naive spatial noise injection would degrade the entire image uniformly, destroying valid structural information.
By contrast, operating in the frequency domain allows us to strictly enforce the AWGN energy profile by filling only the ``spectral valleys'' (missing energy) with synthetic noise, while preserving the ``spectral peaks'' (dominant energy) that correspond to the valid semantic guidance from the geometric update $\mathbf{v}^{(k+1)}$.

\paragraph{Procedural Overview.}
To realize this frequency-selective adaptation, the SH module $\mathcal{T}_{\text{SH}}$ executes a three-stage process:
(1) \textbf{Diagnosis}: We first estimate the Power Spectral Density (PSD) of the current optimization residuals to identify spectral deficits;
(2) \textbf{Synthesis}: We construct a complementary noise field that is spectrally orthogonal to the artifacts yet statistically consistent with the target noise level;
(3) \textbf{Fusion}: The synthetic noise is seamlessly merged with the dual-shifted input to produce a homogenized state $\tilde{\mathbf{v}}^{(k+1)}$.
We detail these steps below.

\paragraph{Step 1: Spectral Density Estimation of the Residual.}
Since the ground truth noise realization within $\mathbf{v}^{(k+1)}$ is inaccessible, we employ a \textit{bootstrap} strategy to estimate the spectral deficit of the current iterate.
We define the empirical residual field $\mathbf{r}^{(k+1)}$ using the posterior mean estimate from the previous iteration, $\mathbf{z}^{(k)}$, as a proxy for the clean signal:
\begin{equation}
    \mathbf{r}^{(k+1)} := \mathbf{v}^{(k+1)} - \mathbf{z}^{(k)}.
\end{equation}
Intuitively, $\mathbf{r}^{(k+1)}$ captures the structured deviations caused by the physics-based update and the dual accumulation.

Then we analyze the spectral characteristics of $\mathbf{r}$ via its PSD, denoted as $S_{\mathbf{r}}(\boldsymbol{\omega})$. To mitigate variance, we formulate the PSD estimation as a non-parametric estimation problem. We compute the smoothed spectral estimate $\hat{S}_{\mathbf{r}}$ via kernel density estimation in the frequency domain:
\begin{equation}
    \hat{S}_{\mathbf{r}}(\boldsymbol{\omega}) := \left( |\mathcal{F}(\mathbf{r})(\boldsymbol{\omega})|^2 \right) * \mathcal{K}_\delta,
    \label{eq:psd_est}
\end{equation}
where $\mathcal{F}$ is the discrete Fourier transform, $\mathcal{K}_\delta$ is a localized smoothing kernel with bandwidth $\delta$. This operation extracts the \textit{spectral signature} of the solver-induced artifacts, identifying frequency bands where the dual variable accumulates excessive structured energy.

\paragraph{Step 2: Complementary Noise Injection.}
Our objective is to construct a rectified input $\tilde{\mathbf{v}}$ whose effective noise spectrum is theoretically flat at the noise level $\sigma_t^2$. We define the \textit{spectral deficit} $\Delta S(\boldsymbol{\omega})$ as the gap between the target white spectrum and the current residual spectrum:
\begin{equation}
    \Delta S(\boldsymbol{\omega}) := \max \left( \epsilon, \, \sigma_t^2 (HW) - \hat{S}_{\mathbf{r}}(\boldsymbol{\omega}) \right),
\end{equation}
where $HW$ is the number of pixels under the unnormalized FFT convention, and $\epsilon$ is a small floor that prevents zero-power frequencies and stabilizes the square-root operation.

To fill this deficit, we introduce a \textit{complementary stochastic process} $\boldsymbol{\xi}$. Unlike standard additive noise, $\boldsymbol{\xi}$ is constructed to be spectrally orthogonal to the dominant artifact modes. We synthesize $\boldsymbol{\xi}$ by modulating a standard white Gaussian process $\mathbf{n} \sim \mathcal{N}(\mathbf{0}, \mathbf{I})$:
\begin{equation}
    \boldsymbol{\xi}^{(k+1)} := \mathcal{F}^{-1} \left( \sqrt{\Delta S(\boldsymbol{\omega})} \odot e^{i \angle \mathcal{F}(\mathbf{n})} \right).
    \label{eq:injection}
\end{equation}
Here, the phase $\angle \mathcal{F}(\mathbf{n})$ is taken from the random noise $\mathbf{n}$, ensuring spatial incoherence with the signal. While the amplitude is shaped by $\Delta S(\boldsymbol{\omega})$.

\paragraph{Step 3: Fusion.}
The final spectrally homogenized input to the denoiser is then formed as:
\begin{equation}
    \tilde{\mathbf{v}}^{(k+1)} := \mathbf{v}^{(k+1)} + \boldsymbol{\xi}^{(k+1)}.
\end{equation}
By construction, the effective noise in $\tilde{\mathbf{v}}^{(k+1)}$ exhibits an approximately flat spectrum, mimicking the AWGN conditions $\mathcal{D}_\sigma$ was trained on.

\paragraph{Theoretical Justification: Moment Matching}
The proposed $\mathcal{T}_{\text{SH}}$ guarantees that the input to the diffusion model is statistically well-posed in terms of second-order moments. We formalize this via the following property.

\begin{proposition}[Second-Order Spectral Consistency]
\label{prop:consistency}
Assume the injected phase is uniformly distributed and independent of $\mathbf{r}$. The expected Power Spectral Density of the homogenized effective noise $\mathbf{n}_\text{eff} = \mathbf{r} + \boldsymbol{\xi}$ satisfies the whitening condition:
\begin{equation}
    \mathbb{E}_{\boldsymbol{\xi}} \left[ S_{\mathbf{n}_\text{eff}}(\boldsymbol{\omega}) \right] \approx \sigma^2 (HW), \quad \forall \boldsymbol{\omega}.
\end{equation}
Consequently, the covariance matrix of the effective perturbation approximates a scaled identity matrix, $\text{Cov}(\mathbf{n}_\text{eff}) \approx \sigma^2 \mathbf{I}$, aligning with the AWGN assumption of $\mathcal{D}_\sigma$.
\end{proposition}

\begin{remark}
This mechanism can be interpreted as \textbf{Coherence Breaking}. Physical artifacts are characterized by high spectral coherence (energy concentration) and phase correlation. By injecting the complementary process $\boldsymbol{\xi}$, we "drown out" the structured artifacts in specific frequency bands, effectively masking them as part of a global AWGN realization. This prevents the diffusion model from hallucinating structures based on artifact patterns.
\end{remark}

\begin{algorithm}[h]
    \caption{Dual-Coupled PnP Diffusion Prior with Spectral Homogenization}
    \label{alg:main}
    \begin{algorithmic}[1]
        \STATE {\bfseries Input:} Measurements $\mathbf{y}$, Forward Operator $\mathbf{A}$, Pre-trained Diffusion Denoiser $\mathcal{D}_\theta$, Total Iterations $K$, Noise Schedule $\{{\sigma}_t\}_{t=1}^{T}$.
        \STATE {\bfseries Initialize:} $\mathbf{x}^{(0)} \leftarrow \mathbf{A}^\top \mathbf{y}$, $\mathbf{z}^{(0)} \sim \mathcal{N}(\mathbf{0}, \mathbf{I})$, $\mathbf{u}^{(0)} \leftarrow \mathbf{0}$.

        \FOR{$k = 0$ {\bfseries to} $K-1$}
        \STATE \textcolor{gray}{// Set diffusion noise schedule}
        \STATE $t \leftarrow t_k$ (e.g., linear spacing from $T$ to $1$)

        \STATE \;\textcolor[RGB]{78, 95, 170}{\(\triangleright\) \textit{1. Data Fidelity Optimization}}
        \STATE $\mathbf{x}^{(k+1)} \leftarrow \underset{\mathbf{x}}{\arg\min} \|\mathbf{A}\mathbf{x} - \mathbf{y}\|_2^2 + {\rho}\|\mathbf{x} - \mathbf{z}^{(k)} + \mathbf{u}^{(k)}\|_2^2$

        \STATE \;\textcolor[RGB]{237, 125, 49}{\(\triangleright\) \textit{2. Spectral Homogenization}}
        \STATE $\mathbf{v}^{(k+1)} \leftarrow \mathbf{x}^{(k+1)} + \mathbf{u}^{(k)}$
        \STATE $\tilde{\mathbf{v}}^{(k+1)} \leftarrow \text{\textbf{S}pectral\textbf{H}omogenize}(\mathbf{v}^{(k+1)}, \mathbf{z}^{(k)}; \sigma_t)$

        \STATE \;\textcolor[RGB]{204, 68, 62}{\(\triangleright\) \textit{3. Denoising Prediction}}
        \STATE $\mathbf{z}^{(k+1)} \leftarrow \mathcal{D}_\theta(\tilde{\mathbf{v}}^{(k+1)}, t)$

        \STATE \;\textcolor[RGB]{89,179,80}{\(\triangleright\) \textit{4. Dual Variable Update}}
        \STATE $\mathbf{u}^{(k+1)} \leftarrow \mathbf{u}^{(k)} + (\mathbf{x}^{(k+1)} - \mathbf{z}^{(k+1)})$
        \ENDFOR
        \STATE {\bfseries Return} $\mathbf{z}^{(K)}$
    \end{algorithmic}
\end{algorithm}

\section{Experiments}
In this section, we validate the effectiveness of the proposed DC-PnPDP framework.
Our experiments aim to answer two key questions:
(1) Does the introduction of the dual variable mitigate steady-state bias and improve convergence under challenging ill-posed problems?
(2) Does Spectral Homogenization successfully suppress the colored artifacts inherent to dual-driven updates, preventing denoiser hallucination?

We evaluate our method on three representative medical imaging inverse problems: (i) Sparse-View (SV) CT Reconstruction, (ii) Limited-Angle (LA) CT Reconstruction, and (iii) Accelerated MRI Reconstruction.
Furthermore, we conduct comprehensive ablation studies to isolate the contributions of the Dual-Variable Coupling and Spectral Homogenization modules.

\subsection{Experimental Settings}

\paragraph{CT Dataset \& Pre-processing.}
We utilize the AbdomenCT-1K dataset~\citep{Ma2021AbdomenCT1K}, which contains 1112 CT scans from 12 medical centers.
We employ 104,534 slices from 219 patients for training the diffusion prior, reserving six distinct patients (2,894 slices) for evaluation.
All slices are resized to 256$\times$256.
We use the \texttt{torch-radon} library to simulate parallel-beam (PB) geometry.
For {Sparse-View CT}, we uniformly sample 20 views from $[0, 180)^\circ$; for {Limited-Angle CT}, we utilize a restricted angular range of $[0, 90]^\circ$ with 90 views.
\textit{Detailed settings are provided in Appendix \ref{app:exp_settings}.}

\paragraph{MRI Dataset \& Pre-processing.} We conduct experiments on the fastMRI brain dataset~\citep{zbontar2018fastMRI}.
We evaluate on a 200-slice test subset following the same single-coil protocol used by the competing diffusion solvers.
The complex-valued images are cropped to $320 \times 320$.
Crucially, we simulate single-coil acquisition to rigorously evaluate the method's performance under extreme undersampling conditions where parallel imaging gains (from coil sensitivity maps) are unavailable.
To simulate accelerated acquisition, we employ 1D equidistant Cartesian masks with acceleration factors of $6\times$ and $10\times$.
We choose this equidistant sampling pattern over random subsampling because it reflects realistic clinical deployment.
More importantly, it generates highly coherent artifacts that are substantially more difficult to remove compared to the incoherent artifacts resulting from random sampling~\cite{knoll2019assessment}.
For the proposed Spectral Homogenization in the complex domain, we apply the frequency-domain modulation to the real and imaginary channels separately.

\paragraph{Implementation Details.} We adopt the EDM  framework~\citep{karras2022elucidating} for training the diffusion prior. For CT task, we train the model from scratch on AbdomenCT-1K. For MRI task, all PnPDP methods are evaluated with the same pre-trained MRI diffusion prior provided by \citet{zheng2025inversebench}.
All experiments are conducted on a single NVIDIA RTX 4090 GPU unless otherwise specified.

\paragraph{Methods in Comparison \& Metrics.} We compare DC-PnPDP against a comprehensive suite of baselines, categorized into:
(1) Classical Methods: FBP (for CT) and Zero-filling (for MRI).
(2) SOTA Diffusion Solvers: We select five representative PnPDP methods: DDNM~\citep{wang2022zero}, DDS~\citep{chung2024decomposed}, DiffPIR~\citep{zhu2023denoising}, the recent DAPS~\citep{zhang2025improving}, and SITCOM~\citep{alkhouri2025sitcom}.
It is worth highlighting that \textit{DiffPIR} can be mathematically formulated as a PnP-HQS framework, whereas our approach represents an ADMM framework. The structural difference lies strictly in the dual variable update and the proposed Spectral Homogenization. \textit{To ensure a rigorous comparison, we align the data-consistency operator of our method exactly with that of DiffPIR.} This control ensures that any performance gain is exclusively maximizing the contribution of the dual-coupling mechanism and spectral homogenization module. \textit{Detailed hyperparameter alignments are provided in Appendix \ref{app:baseline_details}.}
Quantitative evaluation employs Peak Signal-to-Noise Ratio (PSNR) and Structural Similarity (SSIM) for data fidelity. Additional LPIPS results are provided in Appendix~\ref{app:controlled_results}.

\begin{figure*}[t]
    \centering
    \includegraphics[width=\textwidth]{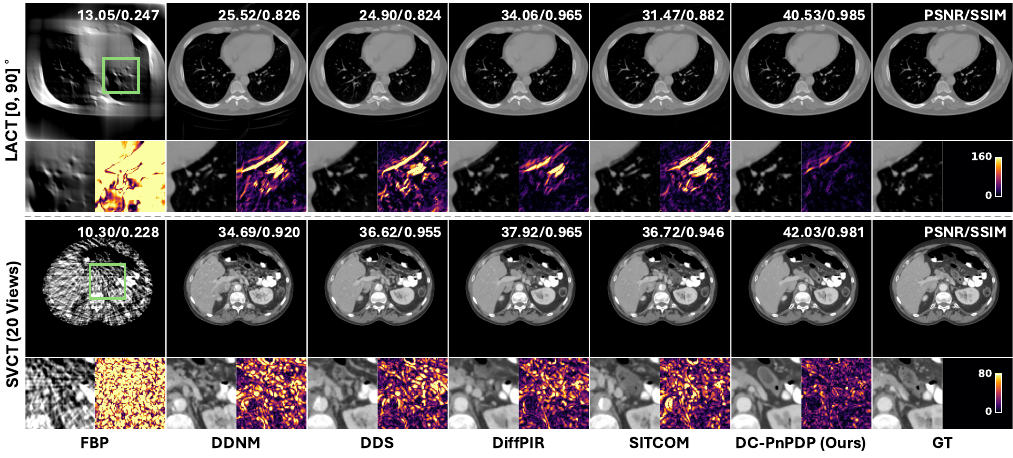}
    \caption{Qualitative comparisons on two CT reconstruction tasks. Visualization windows are set to $[-800, 800]$ HU for LACT and  $[-175, 500]$ HU for SVCT, respectively. }
    \label{fig:Res-CT}
\end{figure*}

\begin{figure*}[t]
    \centering
    \includegraphics[width=\textwidth]{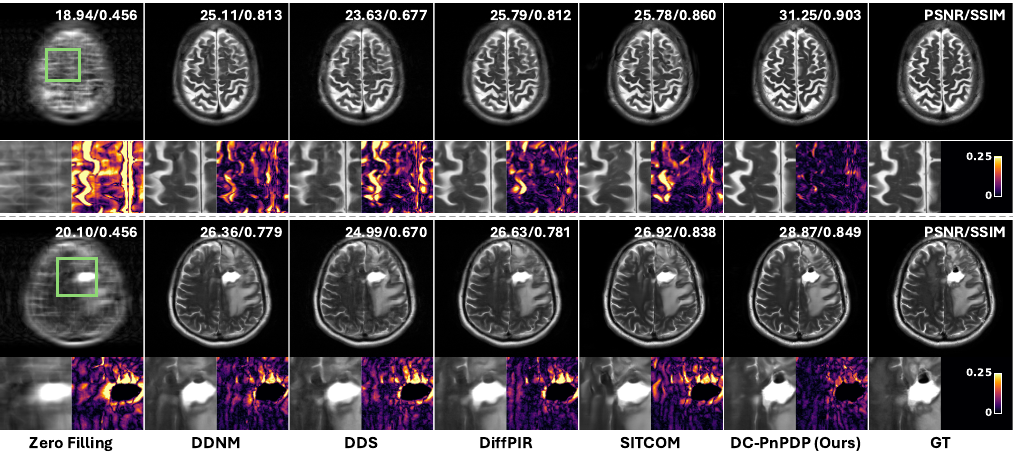}
    \caption{Qualitative comparisons of accelerated fastMRI brain reconstruction at AF = 10, including zoomed-in details and error maps.}
    \label{fig:Res-MRI}
\end{figure*}

\begin{table*}[t]
\centering
\caption{Quantitative comparison under recommended sampling budgets. The table reports PSNR / SSIM on CT and fastMRI brain reconstruction tasks. The best results are highlighted in \textbf{bold}, and the strongest baseline results are \underline{underlined}. The gray 50-NFE DC-PnPDP row is included as an auxiliary budget reference and is not used for the margin row.}
\label{tab:recommended_budget}
\small
\setlength{\tabcolsep}{3.5pt}
\begin{tabular*}{\textwidth}{@{\extracolsep{\fill}}lccccccccc}
\toprule
\multirow{2}{*}{\textbf{Method}} & \multirow{2}{*}{\textbf{NFE}} &
\multicolumn{2}{c}{\textbf{LACT-90}} &
\multicolumn{2}{c}{\textbf{SVCT-20}} &
\multicolumn{2}{c}{\textbf{Brain MRI (AF=6)}} &
\multicolumn{2}{c}{\textbf{Brain MRI (AF=10)}} \\
\cmidrule(lr){3-4}\cmidrule(lr){5-6}\cmidrule(lr){7-8}\cmidrule(lr){9-10}
 & & PSNR $\uparrow$ & SSIM $\uparrow$ & PSNR $\uparrow$ & SSIM $\uparrow$ & PSNR $\uparrow$ & SSIM $\uparrow$ & PSNR $\uparrow$ & SSIM $\uparrow$ \\
\midrule
$\mathbf{A}^{\dagger}\mathbf{y}$ & -- & 16.82 & 0.605 & 21.82 & 0.359 & 23.44 & 0.680 & 20.41 & 0.576 \\
DDNM~\citep{wang2022zero} & 100 & 27.13 & 0.800 & 33.94 & 0.885 & 33.23 & 0.953 & 26.13 & 0.889 \\
DDS~\citep{chung2024decomposed} & 100 & 26.30 & 0.806 & 36.36 & 0.932 & 33.26 & 0.953 & 26.08 & 0.889 \\
DAPS~\citep{zhang2025improving} & $5{\times}200$ & 30.02 & 0.891 & 37.05 & 0.939 & 34.89 & 0.967 & 27.04 & 0.910 \\
SITCOM~\citep{alkhouri2025sitcom} & $50{\times}10$ & 32.07 & 0.911 & 37.76 & 0.945 & \underline{35.58} & \underline{0.969} & \underline{28.67} & \underline{0.927} \\
DiffPIR~\citep{zhu2023denoising} & 100 & \underline{34.70} & \underline{0.926} & \underline{37.86} & \underline{0.947} & 34.88 & 0.965 & 27.92 & 0.918 \\
\midrule
\textcolor{gray}{DC-PnPDP (\textit{Ours})} & \textcolor{gray}{50} & \textcolor{gray}{37.53} & \textcolor{gray}{0.945} & \textcolor{gray}{39.91} & \textcolor{gray}{0.961} & \textcolor{gray}{36.05} & \textcolor{gray}{0.971} & \textcolor{gray}{30.47} & \textcolor{gray}{0.939} \\
DC-PnPDP (\textit{Ours}) & 100 & \textbf{39.46} & \textbf{0.955} & \textbf{40.55} & \textbf{0.963} & \textbf{36.43} & \textbf{0.972} & \textbf{30.91} & \textbf{0.943} \\
\bottomrule
\end{tabular*}
\end{table*}

\subsection{Main Results}

We present the quantitative comparison against SOTA PnP diffusion solvers in Table~\ref{tab:recommended_budget}.
Each baseline is evaluated under the sampling budget recommended by its original work or official implementation, and DC-PnPDP is reported at both 50 and 100 NFEs.
Across CT and fastMRI brain reconstruction tasks, the proposed \textbf{DC-PnPDP} consistently achieves the highest fidelity metrics (PSNR and SSIM).

\paragraph{SVCT and LACT.}
The advantages of our framework are most pronounced in the CT reconstruction tasks.
In the recommended-budget setting, DC-PnPDP at 100 NFEs improves over the strongest baseline by \textbf{+4.76 dB} on LACT-90 (39.46 dB vs. 34.70 dB) and \textbf{+2.69 dB} on SVCT-20 (40.55 dB vs. 37.86 dB).
This empirical evidence strongly validates our theoretical analysis: current PnPDP methods treat the data-consistency and prior sub-problems as loosely coupled steps. Functioning essentially as memoryless P-controllers, they lack the capacity to track the discrepancy between the likelihood and prior terms, often causing the optimization trajectory to stall with significant steady-state errors.
In contrast, our dual-coupled mechanism acts as an Integral controller that strictly enforces the consensus constraint ($\mathbf{x}=\mathbf{z}$), effectively eliminating the bias caused by missing measurements and enhancing geometric accuracy.

Visual inspection of Figure~\ref{fig:Res-CT} corroborates these quantitative gains.
In the LACT task (top row), baselines such as DiffPIR and DAPS struggle to recover the bone structures within the missing wedge direction, resulting in severe blurring and streak artifacts (see zoomed-in ROI). Our DC-PnPDP, however, reconstructs sharp anatomical boundaries that are visually nearly identical to the GT.
For the SVCT task (bottom row), we observe that memoryless baselines tend to \textit{over-smooth} small low-contrast structures (e.g., the small lesion highlighted in the ROI).
In contrast, our method preserves subtle tissue contrast and fine textures without hallucinating non-existent features.
The error maps further highlight that DC-PnPDP yields the lowest pixel-wise reconstruction error, demonstrating superior fidelity across the entire field of view.

\paragraph{Accelerated MRI.}
On fastMRI brain reconstruction, DC-PnPDP also improves the strongest baseline under both acceleration factors.
At AF=6, DC-PnPDP at 100 NFEs reaches 36.43 dB / 0.972 SSIM, improving over SITCOM by +0.85 dB.
At the more challenging AF=10 setting, it reaches 30.91 dB / 0.943 SSIM, improving over SITCOM by +2.24 dB.
These gains indicate that the dual-coupled mechanism is helpful when coherent undersampling artifacts dominate the residual structure.
As shown in Figure~\ref{fig:Res-MRI}, DC-PnPDP yields sharper anatomical boundaries and lower error maps under severe acceleration.

\subsection{Ablation Study}

\paragraph{Effectiveness of Dual-Coupling \& SH.}
We verify the distinct roles of the proposed modules by progressively incorporating them into the framework. The quantitative results are summarized in Table~\ref{tab:ablation_components}.
\textbf{(1)} The \textbf{Base PnP-HQS} configuration (mathematically equivalent to DiffPIR) yields the lowest fidelity (31.36 dB), confirming that memoryless controllers suffer from inherent steady-state bias.
\textbf{(2)} Applying \textbf{SH Only} yields marginal gains (+0.15 dB). This result is structurally revealing: it confirms that the specific spectral pathology (colored artifacts) is primarily induced by the dual accumulation mechanism; in the absence of the dual variable, the distribution shift is less severe, rendering spectral correction less critical.
\textbf{(3)} Conversely, reintroducing the \textbf{Dual Variable} drives the primary performance leap (+4.55 dB) by correcting the geometric bias, yet it still lags behind the full method by $\sim$1.1 dB due to the unmitigated spectral mismatch between the solver residuals and the denoiser's manifold.
\textbf{(4)} Finally, the \textbf{full DC-PnPDP} achieves optimal performance by synergizing integral action (for geometric rigor) with spectral homogenization (for statistical alignment), effectively resolving the bias-hallucination trade-off.

\begin{table}[t]
\centering
\caption{Component ablation results on LACT-90. The first row corresponds to the standard PnP-HQS baseline, i.e., DiffPIR.}
\label{tab:ablation_components}
\small
\setlength{\tabcolsep}{3pt}
\begin{tabular*}{\columnwidth}{@{\extracolsep{\fill}}ccccc}
\toprule
\multicolumn{2}{c}{\textbf{Components}} & \multicolumn{3}{c}{\textbf{Metrics}} \\
\cmidrule(lr){1-2} \cmidrule(lr){3-5}
\textbf{DC} & \textbf{SH} & \textbf{PSNR} $\uparrow$ & \textbf{SSIM} $\uparrow$ & \textbf{LPIPS} $\downarrow$ \\
\midrule
\textcolor[RGB]{214,39,40}{\ding{55}} & \textcolor[RGB]{214,39,40}{\ding{55}} & 31.36 & 0.894 & 0.023 \\
\textcolor[RGB]{214,39,40}{\ding{55}} & \textcolor[RGB]{44,160,44}{\ding{51}} & 31.51 & 0.898 & 0.022 \\
\textcolor[RGB]{44,160,44}{\ding{51}} & \textcolor[RGB]{214,39,40}{\ding{55}} & 35.91 & 0.934 & 0.012 \\
\textcolor[RGB]{44,160,44}{\ding{51}} & \textcolor[RGB]{44,160,44}{\ding{51}} & \textbf{37.02} & \textbf{0.943} & \textbf{0.011} \\
\bottomrule
\end{tabular*}
\end{table}

\begin{figure}[t]
    \centering

    \begin{subfigure}[b]{1.0\columnwidth}
        \centering
        \includegraphics[width=1\columnwidth]{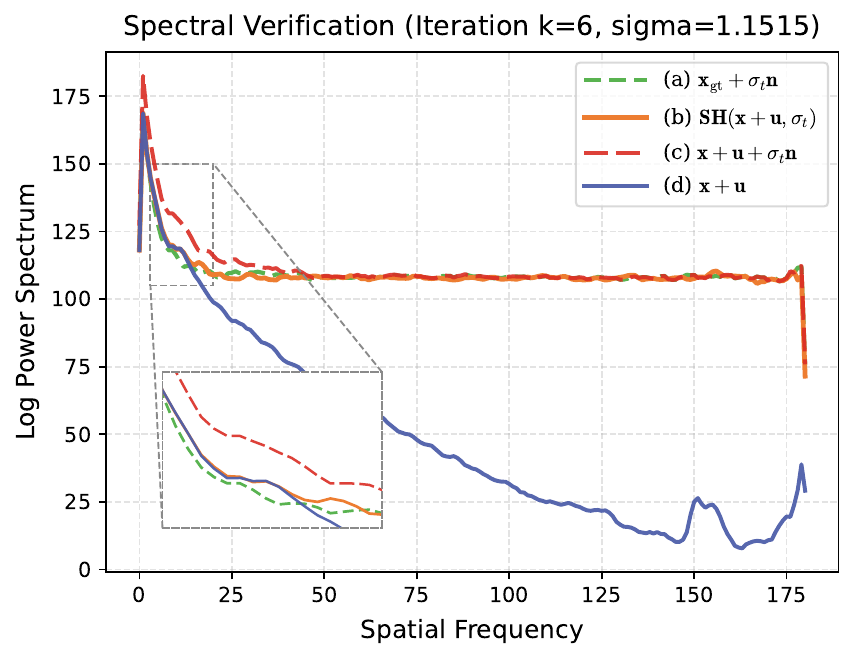}
        \caption{\textbf{Spectral Analysis (PSD):} Visualizing the ``Spectral Deficit''.}
        \label{fig:sub_psd}
    \end{subfigure}

    \vspace{0.2cm}

    \begin{subfigure}[b]{1.0\columnwidth}
        \centering
        \includegraphics[width=1\columnwidth]{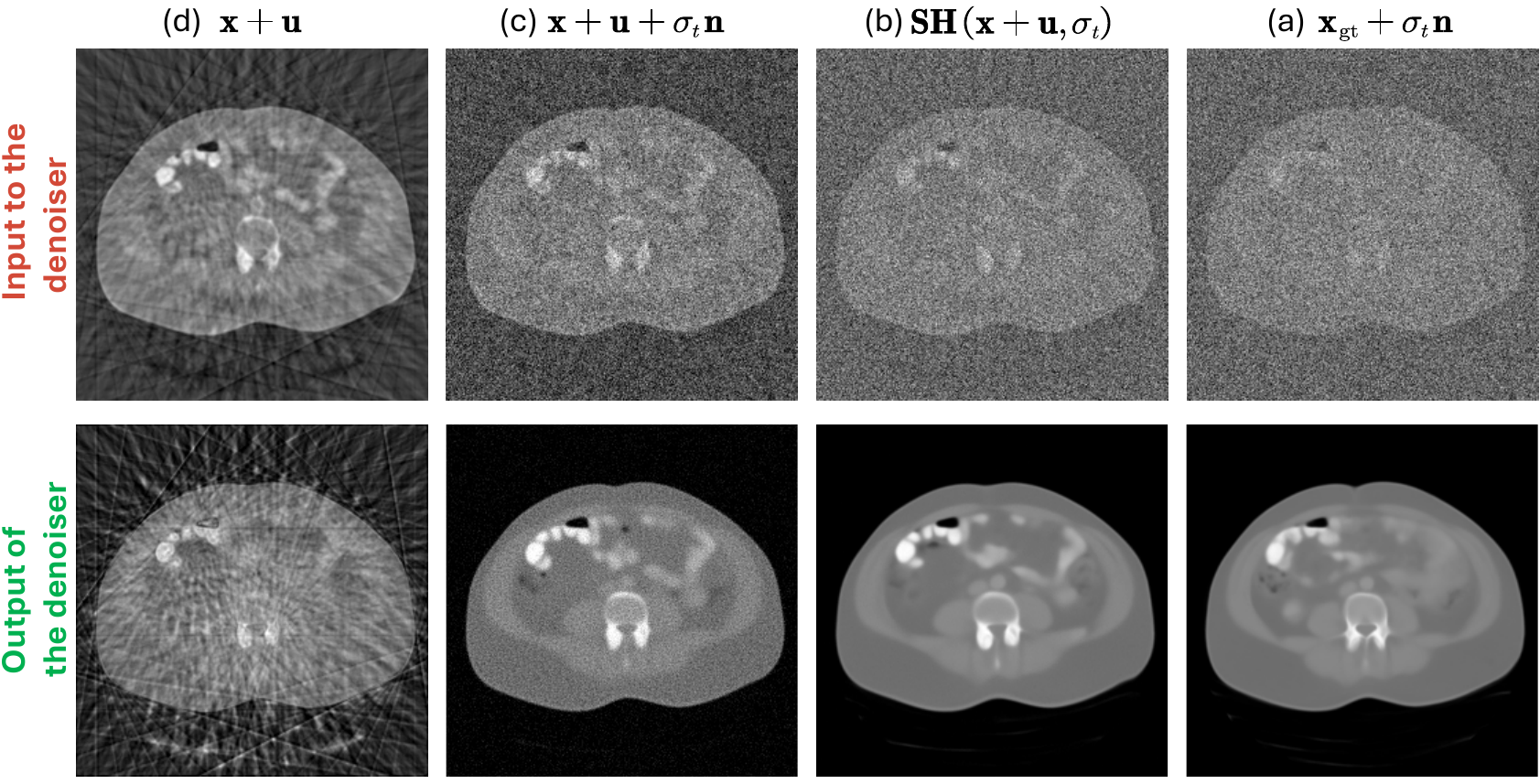}
        \caption{\textbf{Denoising Outcomes:} Comparison of denoiser inputs/outputs.}
        \label{fig:sub_visual}
    \end{subfigure}

    \caption{\textbf{Mechanism Verification: Why Spectral Homogenization Matters.}
        We analyze the spectral characteristics (\textbf{Top}) and the corresponding denoising outputs (\textbf{Bottom}) at iteration $k=6$.
        \textbf{(a) Ideal Reference (\textcolor[RGB]{44,160,44}{Green}):} The on-manifold input ($x_{gt}+\sigma_t n$).
        \textbf{(b) Ours (\textcolor[RGB]{255,127,14}{Orange}):} Our SH module aligns the spectrum perfectly, yielding clean results.
        \textbf{(c) Naive Injection (\textcolor[RGB]{214,39,40}{Red}):} Adding noise ($x+u+\sigma_t n$) creates an \textit{over-energized} spectrum, leading to over-smoothing.
        \textbf{(d) No Injection (\textcolor[RGB]{31,119,180}{Blue}):} The raw dual input ($x+u$) causes OOD hallucinations.}
    \label{fig:spectral_mechanism}
\end{figure}

\paragraph{Effectiveness of Spectral Homogenization on Denoising.}
To physically validate the mechanism of Spectral Homogenization, we analyze the frequency characteristics of the denoiser input. Figure~\ref{fig:spectral_mechanism} visualizes the spectral density and the corresponding single-step denoising output at an intermediate iteration.
We identify two distinct failure modes in baselines:
(1) The \textbf{raw dual input} (Line \textbf{d}, Blue) exhibits strong spectral coloring with high-frequency spikes, an OOD state that directly triggers artifact hallucination.
(2) The \textbf{naive injection} strategy (Line \textbf{c}, Red) results in an \textit{over-energized} spectrum (Red $>$ Green). This creates a fundamental \textbf{variance mismatch} where the effective input noise exceeds the scheduled $\sigma_t$, causing the denoiser to under-correct and leave residual noise.
In contrast, our \textbf{SH module $\mathcal{T}_{\text{SH}}$} (Line \textbf{b}, Orange) precisely fills the spectral deficits (``valleys'') without overshooting. The resulting PSD perfectly aligns with the \textbf{ideal on-manifold reference} (Line \textbf{a}, Green), rendering the input statistically indistinguishable from the training distribution. Visual results confirm that this alignment successfully activates the optimal performance of the diffusion prior, yielding a clean restoration consistent with the ground truth.

\paragraph{Inference Efficiency.}
We evaluate inference efficiency by tracking reconstruction quality as a function of the number of denoiser function evaluations (NFE).
Figure~\ref{fig:nfe_analysis} reports PSNR curves on both SVCT-20 and LACT-90 over a broad NFE range.
DC-PnPDP reaches high-fidelity reconstructions with substantially fewer denoiser calls than the strongest memoryless PnP diffusion baseline, DiffPIR.
In particular, DC-PnPDP at \textbf{30 NFEs} already surpasses DiffPIR at \textbf{100 NFEs}, corresponding to an approximately \textbf{3.3$\times$} reduction in denoiser evaluations at matched or better reconstruction quality.
The full-range curves further show that this advantage is not merely an early-stopping effect.
On SVCT-20, DiffPIR requires up to \textbf{1000 NFEs} to approach the quality achieved by DC-PnPDP with only \textbf{50 NFEs}; on LACT-90, the corresponding comparison is approximately \textbf{1000 NFEs} for DiffPIR versus \textbf{200 NFEs} for DC-PnPDP.
Although the performance gap narrows as the denoiser budget increases to 1000 NFEs, DC-PnPDP remains consistently above the baseline across the entire range, suggesting that the dual-coupled feedback improves both inference efficiency and the final steady-state solution.

\begin{figure}[t]
    \centering
    \includegraphics[width=1\columnwidth]{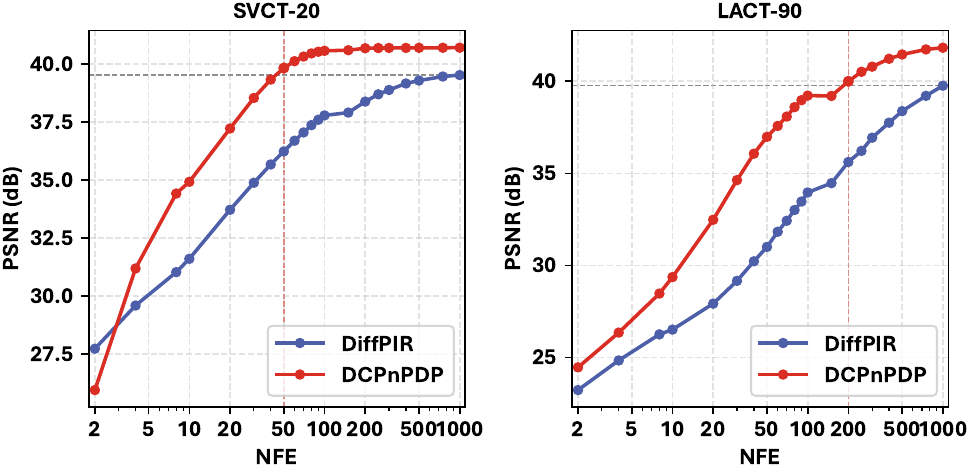}
    \caption{Inference-efficiency comparison on SVCT-20 and LACT-90. PSNR is plotted as a function of NFE to compare reconstruction quality under different denoiser-evaluation budgets. DC-PnPDP reaches higher fidelity with fewer denoiser calls and maintains a stronger high-budget plateau than DiffPIR.}
    \label{fig:nfe_analysis}
\end{figure}

\section{Conclusion and Discussion}
\label{sec:conclusion}

In this work, we proposed \textbf{DC-PnP Diffusion Prior}, a framework that pioneers the integration of the \textit{dual variable}—a cornerstone of classical ADMM optimization—into the PnP diffusion paradigm.
Our primary motivation was to bridge the gap between rigorous mathematical optimization and modern generative priors, injecting the proven convergence properties and geometric precision of traditional solvers into diffusion model based reconstruction.
We identified that this integration, while theoretically sound, is practically non-trivial: the reintroduction of the dual variable inevitably triggers structured spectral artifacts (OOD inputs) that disrupt the diffusion model.
To resolve this, we proposed \textit{Spectral Homogenization (SH)}, a tailored mechanism that adapts these dual-driven residuals into a statistically valid pseudo-white noise, thereby enabling the safe and effective utilization of the dual mechanism.
Extensive experiments demonstrate that our approach achieves SOTA fidelity and superior inference efficiency.

Broadly, our work offers a new perspective on merging \textbf{traditional optimization wisdom} with \textbf{deep generative models}. The key message is not simply to replace classical solvers with learned priors, but to preserve the mathematical structures that make classical optimization reliable and adapt them to the statistical interface required by modern diffusion models. We hope this framework serves as a stepping stone for future research on revitalizing classical algorithmic mechanisms, such as integral action, momentum, adaptive penalties, or primal-dual dynamics, to enhance the fidelity, robustness, and efficiency of data-driven priors in large-scale inverse problems.

\newpage

\section*{Acknowledgements}

This work was supported by the National Key Research and Development Program of China (2024YFC2421100) and the National Natural Science Foundation of China (Grant No. 62571328), and in part by the National Natural Science Foundation of China (Grant No. W2431046), the MoE Key Lab of Intelligent Perception and Human--Machine Collaboration at ShanghaiTech University, and the Shanghai Frontiers Science Center of Human-centered Artificial Intelligence.

\section*{Impact Statement}
This paper presents work whose goal is to advance the field of Machine Learning. There are many potential societal consequences of our work, none which we feel must be specifically highlighted here.

\bibliography{example_paper}

@inproceedings{alkhouri2025sitcom,
  title     = {{SITCOM}: Step-wise Triple-Consistent Diffusion Sampling for Inverse Problems},
  author    = {Alkhouri, Ismail and Liang, Shijun and Huang, Cheng-Han and Dai, Jimmy and Qu, Qing and Ravishankar, Saiprasad and Wang, Rongrong},
  booktitle = {Proc. Int. Conf. Mach. Learn. (ICML)},
  pages     = {1128--1158},
  volume    = {267},
  series    = {Proceedings of Machine Learning Research},
  publisher = {PMLR},
  year      = {2025}
}

@article{bendel2025solving,
  title   = {Solving Inverse Problems using Diffusion with Iterative Colored Renoising},
  author  = {Bendel, Matthew C. and Shastri, Saurav K. and Ahmad, Rizwan and Schniter, Philip},
  journal = {Trans. Mach. Learn. Res.},
  year    = {2025}
}

@book{bequette2003process,
  title     = {Process control: modeling, design, and simulation},
  author    = {Bequette, B Wayne},
  year      = {2003},
  publisher = {Prentice Hall Professional}
}

@article{boyd2011distributed,
  title     = {Distributed optimization and statistical learning via the alternating direction method of multipliers},
  author    = {Boyd, Stephen and Parikh, Neal and Chu, Eric and Peleato, Borja and Eckstein, Jonathan and others},
  journal   = {Found. Trends Mach. Learn.},
  volume    = {3},
  number    = {1},
  pages     = {1--122},
  year      = {2011},
  publisher = {Now Publishers, Inc.}
}

@article{chan2016plug,
  title     = {Plug-and-play ADMM for image restoration: Fixed-point convergence and applications},
  author    = {Chan, Stanley H and Wang, Xiran and Elgendy, Omar A},
  journal   = {IEEE Trans. Comput. Imag.},
  volume    = {3},
  number    = {1},
  pages     = {84--98},
  year      = {2016},
  publisher = {IEEE}
}

@inproceedings{chu2025split,
  title     = {Split Gibbs Discrete Diffusion Posterior Sampling},
  author    = {Chu, Wenda and Wu, Zihui and Chen, Yifan and Song, Yang and Yue, Yisong},
  booktitle = {Adv. Neural Inf. Process. Syst. (NeurIPS)},
  volume    = {38},
  pages     = {145122--145153},
  year      = {2025}
}

@inproceedings{chung2022diffusion,
  title     = {Diffusion Posterior Sampling for General Noisy Inverse Problems},
  author    = {Chung, Hyungjin and Kim, Jeongsol and Mccann, Michael Thompson and Klasky, Marc Louis and Ye, Jong Chul},
  booktitle = {Proc. Int. Conf. Learn. Represent. (ICLR)},
  year      = {2023}
}

@inproceedings{chung2022improving,
  title   = {Improving diffusion models for inverse problems using manifold constraints},
  author  = {Chung, Hyungjin and Sim, Byeongsu and Ryu, Dohoon and Ye, Jong Chul},
  booktitle = {Adv. Neural Inf. Process. Syst. (NeurIPS)},
  volume  = {35},
  pages   = {25683--25696},
  year    = {2022}
}

@inproceedings{chung2024decomposed,
  title     = {Decomposed diffusion sampler for accelerating large-scale inverse problems},
  author    = {Chung, Hyungjin and Lee, Suhyeon and Ye, Jong Chul},
  booktitle = {Proc. Int. Conf. Learn. Represent. (ICLR)},
  year      = {2024}
}

@article{coeurdoux2024plug,
  title   = {Plug-and-Play Split Gibbs Sampler: Embedding Deep Generative Priors in Bayesian Inference},
  author  = {Coeurdoux, Florentin and Dobigeon, Nicolas and Chainais, Pierre},
  journal = {IEEE Trans. Image Process.},
  volume  = {33},
  pages   = {3496--3509},
  year    = {2024}
}

@inproceedings{dhariwal2021diffusion,
  title   = {Diffusion models beat gans on image synthesis},
  author  = {Dhariwal, Prafulla and Nichol, Alexander},
  booktitle = {Adv. Neural Inf. Process. Syst. (NeurIPS)},
  volume  = {34},
  pages   = {8780--8794},
  year    = {2021}
}

@article{du2024dper,
  title   = {DPER: Diffusion prior driven neural representation for limited angle and sparse view CT reconstruction},
  author  = {Du, Chenhe and Lin, Xiyue and Wu, Qing and Tian, Xuanyu and Su, Ying and Luo, Zhe and Zheng, Rui and Chen, Yang and Wei, Hongjiang and Zhou, S Kevin and others},
  journal = {arXiv preprint arXiv:2404.17890},
  year    = {2024}
}

@inproceedings{ho2020denoising,
  title   = {Denoising diffusion probabilistic models},
  author  = {Ho, Jonathan and Jain, Ajay and Abbeel, Pieter},
  booktitle = {Adv. Neural Inf. Process. Syst. (NeurIPS)},
  volume  = {33},
  pages   = {6840--6851},
  year    = {2020}
}

@article{hou2022truncated,
  title   = {Truncated Residual Based Plug-and-Play {ADMM} Algorithm for {MRI} Reconstruction},
  author  = {Hou, Ruizhi and Li, Fang and Zhang, Guixu},
  journal = {IEEE Trans. Comput. Imag.},
  volume  = {8},
  pages   = {96--108},
  year    = {2022}
}

@inproceedings{hurault2022proximal,
  title        = {Proximal Denoiser for Convergent Plug-and-Play Optimization with Nonconvex Regularization},
  author       = {Hurault, Samuel and Leclaire, Arthur and Papadakis, Nicolas},
  booktitle    = {Proc. Int. Conf. Mach. Learn. (ICML)},
  pages        = {9483--9505},
  volume       = {162},
  series       = {Proceedings of Machine Learning Research},
  publisher    = {PMLR},
  year         = {2022}
}

@article{kamilov2023plug,
  title     = {Plug-and-play methods for integrating physical and learned models in computational imaging: Theory, algorithms, and applications},
  author    = {Kamilov, Ulugbek S and Bouman, Charles A and Buzzard, Gregery T and Wohlberg, Brendt},
  journal   = {IEEE Signal Process. Mag.},
  volume    = {40},
  number    = {1},
  pages     = {85--97},
  year      = {2023},
  publisher = {IEEE}
}

@inproceedings{karras2022elucidating,
  title   = {Elucidating the design space of diffusion-based generative models},
  author  = {Karras, Tero and Aittala, Miika and Aila, Timo and Laine, Samuli},
  booktitle = {Adv. Neural Inf. Process. Syst. (NeurIPS)},
  volume  = {35},
  pages   = {26565--26577},
  year    = {2022}
}

@article{kim2025dual,
  title   = {Dual Ascent Diffusion for Inverse Problems},
  author  = {Kim, Minseo and Levy, Axel and Wetzstein, Gordon},
  journal = {arXiv preprint arXiv:2505.17353},
  year    = {2025}
}

@article{knoll2019assessment,
  title     = {Assessment of the generalization of learned image reconstruction and the potential for transfer learning},
  author    = {Knoll, Florian and Hammernik, Kerstin and Kobler, Erich and Pock, Thomas and Recht, Michael P and Sodickson, Daniel K},
  journal   = {Magn. Reson. Med.},
  volume    = {81},
  number    = {1},
  pages     = {116--128},
  year      = {2019},
  publisher = {Wiley Online Library}
}

@article{Ma2021AbdomenCT1K,
  author  = {Ma, Jun and Zhang, Yao and Gu, Song and Zhu, Cheng and Ge, Cheng and Zhang, Yichi and An, Xingle and Wang, Congcong and Wang, Qiyuan and Liu, Xin and Cao, Shucheng and Zhang, Qi and Liu, Shangqing and Wang, Yunpeng and Li, Yuhui and He, Jian and Yang, Xiaoping},
  journal = {IEEE Trans. Pattern Anal. Mach. Intell.},
  title   = {AbdomenCT-1K: Is Abdominal Organ Segmentation a Solved Problem?},
  year    = {2022},
  volume  = {44},
  number  = {10},
  pages   = {6695--6714}
}

@article{ma2026universal,
  title   = {Universal pre-training for generalizable incomplete-view {CT} reconstruction},
  author  = {Ma, Chenglong and others},
  journal = {Pattern Recognit.},
  volume  = {178},
  pages   = {113513},
  year    = {2026}
}

@inproceedings{park2023convergence,
  title     = {Convergence of Nonconvex {PnP-ADMM} with {MMSE} Denoisers},
  author    = {Park, Chicago and Shoushtari, Shirin and Gan, Weijie and Kamilov, Ulugbek S.},
  booktitle = {Proc. IEEE Int. Workshop Comput. Adv. Multi-Sensor Adapt. Process. (CAMSAP)},
  pages     = {511--515},
  year      = {2023}
}

@article{park2026stochastic,
  title   = {Stochastic Generative Plug-and-Play Priors},
  author  = {Park, Chicago Y. and Chandler, Edward P. and Hu, Yuyang and McCann, Michael T. and Garcia-Cardona, Cristina and Wohlberg, Brendt and Kamilov, Ulugbek S.},
  journal = {arXiv preprint arXiv:2604.03603},
  year    = {2026}
}

@inproceedings{peebles2023scalable,
  title     = {Scalable diffusion models with transformers},
  author    = {Peebles, William and Xie, Saining},
  booktitle = {Proc. IEEE/CVF Int. Conf. Comput. Vis. (ICCV)},
  pages     = {4195--4205},
  year      = {2023}
}

@incollection{robbins1992empirical,
  title     = {An empirical Bayes approach to statistics},
  author    = {Robbins, Herbert E},
  booktitle = {Breakthroughs in Statistics: Foundations and basic theory},
  pages     = {388--394},
  year      = {1992},
  publisher = {Springer}
}

@inproceedings{rombach2022high,
  title     = {High-resolution image synthesis with latent diffusion models},
  author    = {Rombach, Robin and Blattmann, Andreas and Lorenz, Dominik and Esser, Patrick and Ommer, Bj{\"o}rn},
  booktitle = {Proc. IEEE/CVF Conf. Comput. Vis. Pattern Recognit. (CVPR)},
  pages     = {10684--10695},
  year      = {2022}
}

@inproceedings{ryu2019plug,
  title        = {Plug-and-play methods provably converge with properly trained denoisers},
  author       = {Ryu, Ernest and Liu, Jialin and Wang, Sicheng and Chen, Xiaohan and Wang, Zhangyang and Yin, Wotao},
  booktitle    = {Proc. Int. Conf. Mach. Learn. (ICML)},
  pages        = {5546--5557},
  volume       = {97},
  series       = {Proceedings of Machine Learning Research},
  publisher    = {PMLR},
  year         = {2019}
}

@inproceedings{shoushtari2024prior,
  title        = {Prior Mismatch and Adaptation in {PnP-ADMM} with a Nonconvex Convergence Analysis},
  author       = {Shoushtari, Shirin and Liu, Jiaming and Chandler, Edward P. and Asif, M. Salman and Kamilov, Ulugbek S.},
  booktitle    = {Proc. Int. Conf. Mach. Learn. (ICML)},
  pages        = {45154--45182},
  volume       = {235},
  series       = {Proceedings of Machine Learning Research},
  publisher    = {PMLR},
  year         = {2024}
}

@inproceedings{shrestha2026taming,
  title     = {Taming Score-Based Denoisers in {ADMM}: A Convergent Plug-and-Play Framework},
  author    = {Shrestha, Rajesh and Fu, Xiao},
  booktitle = {Proc. Int. Conf. Learn. Represent. (ICLR)},
  year      = {2026}
}

@inproceedings{song2020score,
  title     = {Score-based generative modeling through stochastic differential equations},
  author    = {Song, Yang and Sohl-Dickstein, Jascha and Kingma, Diederik P and Kumar, Abhishek and Ermon, Stefano and Poole, Ben},
  booktitle = {Proc. Int. Conf. Learn. Represent. (ICLR)},
  year      = {2021}
}

@inproceedings{song2021solving,
  title     = {Solving Inverse Problems in Medical Imaging with Score-Based Generative Models},
  author    = {Song, Yang and Shen, Liyue and Xing, Lei and Ermon, Stefano},
  booktitle = {Proc. Int. Conf. Learn. Represent. (ICLR)},
  year      = {2022}
}

@inproceedings{song2024solving,
  title     = {Solving Inverse Problems with Latent Diffusion Models via Hard Data Consistency},
  author    = {Song, Bowen and Kwon, Soo Min and Zhang, Zecheng and Hu, Xinyu and Qu, Qing and Shen, Liyue},
  booktitle = {Proc. Int. Conf. Learn. Represent. (ICLR)},
  year      = {2024}
}

@article{tang2020fast,
  title   = {A Fast Stochastic Plug-and-Play {ADMM} for Imaging Inverse Problems},
  author  = {Tang, Junqi and Davies, Mike},
  journal = {arXiv preprint arXiv:2006.11630},
  year    = {2020}
}

@inproceedings{venkatakrishnan2013plug,
  title        = {Plug-and-play priors for model based reconstruction},
  author       = {Venkatakrishnan, Singanallur V and Bouman, Charles A and Wohlberg, Brendt},
  booktitle    = {Proc. IEEE Global Conf. Signal Inf. Process. (GlobalSIP)},
  pages        = {945--948},
  year         = {2013},
  organization = {IEEE}
}

@inproceedings{wang2022zero,
  title     = {Zero-Shot Image Restoration Using Denoising Diffusion Null-Space Model},
  author    = {Wang, Yinhuai and Yu, Jiwen and Zhang, Jian},
  booktitle = {Proc. Int. Conf. Learn. Represent. (ICLR)},
  year      = {2023}
}

@inproceedings{wang2024dmplug,
  title   = {Dmplug: A plug-in method for solving inverse problems with diffusion models},
  author  = {Wang, Hengkang and Zhang, Xu and Li, Taihui and Wan, Yuxiang and Chen, Tiancong and Sun, Ju},
  booktitle = {Adv. Neural Inf. Process. Syst. (NeurIPS)},
  volume  = {37},
  pages   = {117881--117916},
  year    = {2024}
}

@inproceedings{wu2024principled,
  title   = {Principled probabilistic imaging using diffusion models as plug-and-play priors},
  author  = {Wu, Zihui and Sun, Yu and Chen, Yifan and Zhang, Bingliang and Yue, Yisong and Bouman, Katherine L},
  booktitle = {Adv. Neural Inf. Process. Syst. (NeurIPS)},
  volume  = {37},
  pages   = {118389--118427},
  year    = {2024}
}

@inproceedings{xu2024provably,
  title   = {Provably robust score-based diffusion posterior sampling for plug-and-play image reconstruction},
  author  = {Xu, Xingyu and Chi, Yuejie},
  booktitle = {Adv. Neural Inf. Process. Syst. (NeurIPS)},
  volume  = {37},
  pages   = {36148--36184},
  year    = {2024}
}

@inproceedings{yang2023dsg,
  title     = {Guidance with Spherical Gaussian Constraint for Conditional Diffusion},
  author    = {Lingxiao Yang and Shutong Ding and Yifan Cai and Jingyi Yu and Jingya Wang and Ye Shi},
  booktitle = {Proc. Int. Conf. Mach. Learn. (ICML)},
  pages     = {56071--56095},
  volume    = {235},
  series    = {Proceedings of Machine Learning Research},
  publisher = {PMLR},
  year      = {2024}
}

@article{zbontar2018fastMRI,
  title         = {{fastMRI}: An Open Dataset and Benchmarks for Accelerated {MRI}},
  author        = {Jure Zbontar and Florian Knoll and Anuroop Sriram and Tullie Murrell and Zhengnan Huang and Matthew J. Muckley and Aaron Defazio and Ruben Stern and Patricia Johnson and Mary Bruno and Marc Parente and Krzysztof J. Geras and Joe Katsnelson and Hersh Chandarana and Zizhao Zhang and Michal Drozdzal and Adriana Romero and Michael Rabbat and Pascal Vincent and Nafissa Yakubova and James Pinkerton and Duo Wang and Erich Owens and C. Lawrence Zitnick and Michael P. Recht and Daniel K. Sodickson and Yvonne W. Lui},
  journal       = {arXiv preprint arXiv:1811.08839},
  archiveprefix = {arXiv},
  eprint        = {1811.08839},
  year          = {2018}
}

@inproceedings{zhang2025decoupling,
  title     = {Decoupling Training-Free Guided Diffusion by {ADMM}},
  author    = {Zhang, Youyuan and Liu, Zehua and Li, Zenan and Li, Zhaoyu and Clark, James J. and Si, Xujie},
  booktitle = {Proc. IEEE/CVF Conf. Comput. Vis. Pattern Recognit. (CVPR)},
  pages     = {23292--23302},
  year      = {2025}
}

@inproceedings{zhang2025improving,
  title     = {Improving diffusion inverse problem solving with decoupled noise annealing},
  author    = {Zhang, Bingliang and Chu, Wenda and Berner, Julius and Meng, Chenlin and Anandkumar, Anima and Song, Yang},
  booktitle = {Proc. IEEE/CVF Conf. Comput. Vis. Pattern Recognit. (CVPR)},
  pages     = {20895--20905},
  year      = {2025}
}

@inproceedings{zheng2025inversebench,
  title     = {InverseBench: Benchmarking Plug-and-Play Diffusion Priors for Inverse Problems in Physical Sciences},
  author    = {Zheng, Hongkai and Chu, Wenda and Zhang, Bingliang and Wu, Zihui and Wang, Austin and Feng, Berthy and Zou, Caifeng and Sun, Yu and Kovachki, Nikola Borislavov and Ross, Zachary E and others},
  booktitle = {Proc. Int. Conf. Learn. Represent. (ICLR)},
  year      = {2025}
}

@inproceedings{zhu2023denoising,
  title     = {Denoising Diffusion Models for Plug-and-Play Image Restoration},
  author    = {Zhu, Yuanzhi and Zhang, Kai and Liang, Jingyun and Cao, Jiezhang and Wen, Bihan and Timofte, Radu and Van Gool, Luc},
  booktitle = {Proc. IEEE/CVF Conf. Comput. Vis. Pattern Recognit. (CVPR)},
  pages     = {1219--1229},
  year      = {2023}
}
\bibliographystyle{icml2026}

\newpage
\appendix
\onecolumn
\section{Appendix.}

\subsection{Detailed Algorithm}
\begin{algorithm}[!htbp]
    \caption{Detailed Implementation: Dual-Coupled PnP Diffusion}
    \label{alg:detailed_implementation}
    \begin{algorithmic}[1]
        \STATE {\bfseries Input:} Measurements $\mathbf{y}$, Forward Operator $\mathbf{A}$, Pre-trained Denoiser $\mathcal{D}_\theta$.
        \STATE {\bfseries Hyperparameters:}
        \STATE \quad Total Iterations $K$, Penalty Parameter $\rho$.
        \STATE \quad Noise Schedule $\{\sigma_t\}_{t=1}^{T}$ (aligned with denoiser training).
        \STATE \quad SH Smoothing Kernel $\mathcal{K}_w$ (e.g., Gaussian with window size $w=7$).
        \STATE \quad PSD Floor $\epsilon$.
        \STATE {\bfseries Initialize:} $\mathbf{x}^{(0)} \leftarrow \mathbf{A}^\top \mathbf{y}$, $\mathbf{z}^{(0)} \sim \mathcal{N}(\mathbf{0}, \mathbf{I})$, $\mathbf{u}^{(0)} \leftarrow \mathbf{0}$.

        \STATE \textcolor{gray}{// Begin Iterative Restoration}
        \FOR{$k = 0$ {\bfseries to} $K-1$}
        \STATE $t \leftarrow t_k$ \COMMENT{Map iteration $k$ to diffusion timestep}

        \STATE \textcolor[RGB]{78, 95, 170}{\(\triangleright\) \textit{1. Data Fidelity Optimization}}
        \STATE $\mathbf{x}^{(k+1)} \leftarrow \underset{\mathbf{x}}{\arg\min} \|\mathbf{A}\mathbf{x} - \mathbf{y}\|_2^2 + {\rho}\|\mathbf{x} - \mathbf{z}^{(k)} + \mathbf{u}^{(k)}\|_2^2$

        \STATE \textcolor[RGB]{237, 125, 49}{\(\triangleright\) \textit{2. Spectral Homogenization}}
        \STATE $\mathbf{v}^{(k+1)} \leftarrow \mathbf{x}^{(k+1)} + \mathbf{u}^{(k)}$ \COMMENT{Dual-shifted input}
        \STATE \textcolor{gray}{// A.Spectral Density Estimation of the Residual}
        \STATE $\mathbf{r} \leftarrow \mathbf{v}^{(k+1)} - \mathbf{z}^{(k)}$
        \STATE $\hat{S}_{\mathbf{r}} \leftarrow \big| \mathcal{F}(\mathbf{r}) \big|^2 * \mathcal{K}_w$ \COMMENT{Smoothed Power Spectral Density}
        \STATE \textcolor{gray}{// B. Spectral Deficit Calculation (Clipping)}
        \STATE $\Delta S \leftarrow \max\left(\epsilon, \sigma_t^2(HW) - \hat{S}_{\mathbf{r}}\right)$ \COMMENT{Target noise level $\sigma_t$ varies with $t$}
        \STATE \textcolor{gray}{// C. Phase Transplantation (Ensures Hermitian Symmetry)}
        \STATE Generate spatial white noise $\mathbf{n} \sim \mathcal{N}(\mathbf{0}, \mathbf{I})$
        \STATE $\mathbf{\Phi} \leftarrow \mathcal{F}(\mathbf{n}) \oslash |\mathcal{F}(\mathbf{n})|$ \COMMENT{Extract random phase structure}
        \STATE $\boldsymbol{\xi} \leftarrow \mathcal{F}^{-1} \left( \sqrt{\Delta S} \odot \mathbf{\Phi} \right)$ \COMMENT{Synthesize complementary noise}
        \STATE \textcolor{gray}{// D. Homogenized Input Construction}
        \STATE $\tilde{\mathbf{v}}^{(k+1)} \leftarrow \mathbf{v}^{(k+1)} + \boldsymbol{\xi}$

        \STATE \textcolor[RGB]{204, 68, 62}{\(\triangleright\) \textit{3. Denoising Prediction}}
        \STATE $\mathbf{z}^{(k+1)} \leftarrow \mathcal{D}_\theta(\tilde{\mathbf{v}}^{(k+1)}, t)$

        \STATE \textcolor[RGB]{89,179,80}{\(\triangleright\) \textit{4. Dual Variable Update}}
        \STATE $\mathbf{u}^{(k+1)} \leftarrow \mathbf{u}^{(k)} + (\mathbf{x}^{(k+1)} - \mathbf{z}^{(k+1)})$
        \ENDFOR

        \STATE {\bfseries Return} $\mathbf{z}^{(K)}$
    \end{algorithmic}
\end{algorithm}

\subsection{Additional Dataset and Diffusion Prior Details}
\label{app:exp_settings}

\paragraph{CT Data, Pre-processing, and Forward Model.}
We utilize the \textbf{AbdomenCT-1K} dataset~\citep{Ma2021AbdomenCT1K} for training and evaluation. To ensure high-quality priors, we strictly screen the dataset and select scans with a slice thickness of $\le 1$ mm. This screening process yields a subset of 225 high-resolution scans. We randomly reserve 6 scans for the test set and use the remaining 219 scans for training.
Consistent with the main text, the Hounsfield Unit (HU) values are clipped to the range $[-800, 800]$ and linearly normalized to $[-1, 1]$.
For the parallel-beam simulation, we employ a 1D flat detector with 363 uniformly spaced bins and a detector pitch of 1.0 mm.

\paragraph{CT Diffusion Prior Training.}
We adopt the EDM framework~\citep{karras2022elucidating} for training the diffusion prior. The network architecture is a \texttt{DDPM++} U-Net~\citep{song2020score} optimized for generating $256 \times 256$ images.
Training was conducted on a cluster of 8 $\times$ NVIDIA A100 80GB GPUs.
We trained the model for approximately 35,000 steps (equivalent to $\sim$37 epochs) with a total batch size of 112. The entire training process took approximately 24 hours.
We used the Adam optimizer with the following hyperparameters: learning rate $lr = 1 \times 10^{-4}$, $\beta_1 = 0.9$, $\beta_2 = 0.999$, and $\epsilon = 10^{-8}$.

\paragraph{MRI Data and Pretrained Diffusion Prior.}
We use the data split and pretrained MRI diffusion weights in \textsc{InverseBench}~\citep{zheng2025inversebench}, which are built on the fastMRI dataset~\citep{zbontar2018fastMRI}. We evaluate two anatomical regions, brain and knee. For accelerated reconstruction, we retrospectively subsample the fully sampled k-space using one-dimensional equispaced Cartesian masks along the phase-encoding direction, with acceleration factors $R=6$ and $R=10$, and reconstruct under a single-coil Fourier encoding model. This setting is intended to isolate the severely undersampled inverse problem, where reconstruction cannot rely on parallel-imaging gains from coil sensitivity maps. Equispaced Cartesian subsampling also induces coherent aliasing artifacts, which are more structured and challenging than the incoherent artifacts produced by random subsampling. Across both anatomical regions, all PnP diffusion solvers use the same pretrained MRI diffusion weights; no method-specific MRI retraining or fine-tuning is performed. For our complex-valued SH module, the frequency-domain modulation is applied independently to the real and imaginary channels.

\subsection{Implementation Details of Baselines}
\label{app:baseline_details}

We provide detailed hyperparameter settings for all competing methods to ensure reproducibility.
When a baseline contains task-dependent numerical settings, we tune them on a reserved 10-slice validation split for each task rather than transferring MRI-specific configurations to CT.

\paragraph{DDNM~\citep{wang2022zero}.}
DDNM relies on a range-null space decomposition. The update rule can be expressed as
\begin{equation}
    \mathbb{E}\!\left[ \mathbf{x}_{0} \mid \mathbf{x}_{t}, \mathbf{y} \right]^{(\text{DDNM})}
    \approx \left( \mathbf{I} - \mathbf{A}^{\dagger}\mathbf{A} \right) D_{\theta^{*}}(\mathbf{x}_{t}) + \mathbf{A}^{\dagger}\mathbf{y},
    \label{eq:ddnm}
\end{equation}
where $\mathbf{A}^{\dagger}$ denotes the pseudo-inverse of the measurement operator $\mathbf{A}$.
For CT reconstruction, $\mathbf{A}^\dagger$ is approximated using the Simultaneous Algebraic Reconstruction Technique (SIRT). We set the number of SIRT iterations to {30} for all experiments.

\paragraph{DDS~\citep{chung2024decomposed}.}
DDS solves the proximal data-consistency sub-problem using the Conjugate Gradient (CG) method. We set the number of CG iterations to \textbf{5}, following the original paper and official implementation.

\paragraph{DAPS~\citep{zhang2025improving}.}
We utilize the official codebase for DAPS and report it under two distinct evaluation protocols.
In the main recommended-budget comparison (Table~\ref{tab:recommended_budget}), DAPS is evaluated with its DAPS-1K setting, i.e., 5 ODE prior steps and 200 annealing/data-consistency steps, corresponding to the $5{\times}200$ budget shown in the table.
In the controlled 50-NFE appendix comparison (Table~\ref{tab:main_results}), we instead match the denoiser budget across diffusion solvers: the prior sampling step is implemented as a 1-step ODE solver (equivalent to Tweedie's denoising formula), while the data-consistency step uses Langevin dynamics with 100 iterations.
The Langevin parameters are selected on the task-specific validation split:
\begin{itemize}
    \item \textbf{LACT-90:} \texttt{langevin\_step\_size}=$4.1{\times}10^{-6}$, \texttt{noise\_level}=$2.5{\times}10^{-1}$, and \texttt{step\_size\_decay}=$1.2{\times}10^{-1}$.
    \item \textbf{SVCT-20:} \texttt{langevin\_step\_size}=$1.3{\times}10^{-7}$, \texttt{noise\_level}=$1.2{\times}10^{-1}$, and \texttt{step\_size\_decay}=$8.3{\times}10^{-2}$.
    \item \textbf{MRI:} \texttt{langevin\_step\_size}=$1.0{\times}10^{-7}$, \texttt{noise\_level}=$1.0{\times}10^{-3}$, and \texttt{step\_size\_decay}=$2.0{\times}10^{-1}$.
\end{itemize}

\paragraph{SITCOM~\citep{alkhouri2025sitcom}.}
SITCOM is included as an additional diffusion-based baseline in the expanded evaluation. We use its recommended 50 outer diffusion steps with 10 inner optimization steps. The learning rate of the inner optimization is set to 0.02 for the LACT and SVCT tasks, and to 0.004 for the MRI tasks.

\paragraph{Controlled Comparison: DiffPIR~\cite{zhu2023denoising} vs. DC-PnPDP (Ours).}
To rigorously isolate the contribution of our proposed modules, we formulate the comparison between DiffPIR~\citep{zhu2023denoising} and DC-PnPDP as a controlled ablation. These two methods are implemented under a strictly unified optimization interface and hyperparameter setting, with the following key points:
\begin{itemize}
    \item \textbf{Formulation:} DiffPIR can be regarded as a standard \textbf{PnP-HQS} solver, while our DC-PnPDP is a \textbf{PnP-ADMM} solver.
    \item \textbf{Data Consistency (DC):} Both methods use the exact same Conjugate Gradient (CG) solver for the data-consistency sub-problem ($\min_\mathbf{x} \|\mathbf{y}-\mathbf{A}\mathbf{x}\|^2 + \rho\|\mathbf{x}-\mathbf{v}\|^2$).
          \begin{itemize}
              \item For \textbf{Sparse-View (SVCT)}, we set CG iterations = 20.
              \item For \textbf{Limited-Angle (LACT)}, we set CG iterations = 100 and the proximal penalty parameter $\rho = 1 \times 10^{-5}\times \frac{1}{\sigma_t^2}$.
          \end{itemize}
    \item \textbf{Key Difference:} Under this strictly controlled setting, the performance gap is exclusively attributed to two factors: (1) the update of the dual variable $\mathbf{u}$ (present in ours, absent in DiffPIR), and (2) the noise injection strategy (Spectral Homogenization in ours vs. standard DDIM-style AWGN in DiffPIR).
\end{itemize}
DC-PnPDP inherits the optimized DiffPIR data-consistency hyperparameters in this controlled comparison and only adds the dual-coupling and SH modules. The DiffPIR comparison is implemented through a PnP-HQS proximal data-consistency step. For single-coil MRI, this subproblem has a closed-form solution; for CT, it is solved by CG. This differs from implementations that use a single gradient step per diffusion iteration and therefore require substantially larger step counts to obtain comparable data-consistency optimization depth.

\section{Additional Experimental Results and Diagnostics}
\label{app:additional_results}

\subsection{Controlled 50-NFE Evaluation}
\label{app:controlled_results}
The main text reports results under the sampling budgets recommended for each competing method, which is the most relevant protocol for assessing final practical performance. We additionally keep the controlled 50-NFE protocol in the appendix because it answers a complementary question: \textit{when all diffusion-based solvers are allowed the same number of denoiser evaluations, how much of the performance difference can be attributed to the optimization interface rather than to additional prior calls?} This comparison is especially useful for isolating the effect of replacing the memoryless PnP-HQS interface with the proposed dual-coupled PnP-ADMM interface.

Table~\ref{tab:main_results} presents the matched-budget results on the CT tasks and fastMRI knee reconstruction. With the denoiser budget fixed at 50 NFEs for all diffusion-based methods, DC-PnPDP remains consistently competitive and often achieves the best reconstruction quality. These results suggest that the improvement arises from the proposed dual-coupled data-consistency mechanism rather than from allocating a larger number of diffusion prior evaluations.

\begin{table}[!htbp]
       \caption{Controlled 50-NFE quantitative comparison on medical image reconstruction tasks, including SVCT with 20 views, LACT with angular range $[0,90]^\circ$, and accelerated fastMRI knee reconstruction at AF = 6 and AF = 10. The best results are highlighted in \textbf{bold} and the second-best are \underline{underlined}. Here, ``$\mathbf{A}^\dagger \mathbf{y}$'' denotes the pseudo-inverse reconstruction, i.e., FBP for CT and zero-filling for MRI.}
    \centering
    \small
    \setlength{\tabcolsep}{4pt}
    \resizebox{\linewidth}{!}{%
        \begin{tabular}{lccc ccc ccc ccc}
            \toprule
            \multirow{2.5}{*}{\textbf{Method}}              &
            \multicolumn{3}{c}{\textbf{LACT-90}}          &
            \multicolumn{3}{c}{\textbf{SVCT-20}}          &
            \multicolumn{3}{c}{\textbf{Accelerated MRI (AF=6)}} &
            \multicolumn{3}{c}{\textbf{Accelerated MRI (AF=10)}}                                                                                                                                                                                                                                                    \\
            \cmidrule(lr){2-4}\cmidrule(lr){5-7}\cmidrule(lr){8-10}\cmidrule(lr){11-13}
                                                          & PSNR $\uparrow$   & SSIM $\uparrow$   & LPIPS $\downarrow$ & PSNR $\uparrow$   & SSIM $\uparrow$   & LPIPS $\downarrow$ & PSNR $\uparrow$   & SSIM $\uparrow$   & LPIPS $\downarrow$ & PSNR $\uparrow$   & SSIM $\uparrow$   & LPIPS $\downarrow$ \\
            \midrule
            $\mathbf{A}^\dagger \mathbf{y}$               & 16.82             & 0.606             & 0.140              & 21.81             & 0.359             & 0.249              & 21.72             & 0.523             & 0.194              & 18.51             & 0.426             & 0.237              \\
            DDNM~\cite{wang2022zero}                                          & 30.62             & 0.887             & 0.025              & 36.22             & 0.931             & 0.021              & 27.95             & \underline{0.763} & 0.043              & \underline{25.32} & \underline{0.681} & 0.065              \\
            DDS~\cite{chung2024decomposed}                                           & 24.86             & 0.775             & 0.056              & \underline{36.71} & \underline{0.938} & 0.021              & 27.12             & 0.737             & \textbf{0.040}     & 24.01             & 0.644             & \textbf{0.061}     \\
            DAPS~\cite{zhang2025improving}                                          & 28.41             & 0.855             & 0.032              & 36.18             & 0.930             & \underline{0.020}  & 27.86             & 0.760             & \underline{0.042}  & 25.24             & 0.677             & \underline{0.062}  \\
            DiffPIR~\cite{zhu2023denoising}                                       & \underline{31.03} & \underline{0.892} & \underline{0.023}  & 36.23             & 0.931             & 0.021              & \underline{27.96} & \underline{0.763} & 0.043              & 25.27             & 0.680             & 0.065              \\
            DC-PnPDP (\textit{Ours})                                      & \textbf{36.98}    & \textbf{0.943}    & \textbf{0.010}     & \textbf{39.81}    & \textbf{0.960}    & \textbf{0.018}     & \textbf{28.50}    & \textbf{0.784}    & 0.046              & \textbf{26.21}    & \textbf{0.708}    & 0.066              \\
            \bottomrule
        \end{tabular}
    }

\label{tab:main_results}

\end{table}

The fastMRI knee results also help explain why the numerical gain on this MRI setting is more modest than on the CT tasks. Knee MRI in fastMRI has relatively low SNR, so pixel-wise metrics such as PSNR can be strongly influenced by residual background noise rather than by anatomical alignment alone. This effect can partially mask the geometric benefits of improved data consistency, even when the reconstruction better preserves tissue boundaries and suppresses coherent aliasing artifacts. Nevertheless, at the more challenging AF = 10 setting, DC-PnPDP still improves the runner-up by nearly 1 dB in PSNR. The qualitative examples in Figure~\ref{fig:controlled_knee_mri} are therefore important complementary evidence: they show sharper anatomical boundaries and stronger suppression of vertical aliasing streaks, consistent with the intended role of the dual-coupled update.

\begin{figure}[htbp]
    \centering
    \includegraphics[width=1\textwidth]{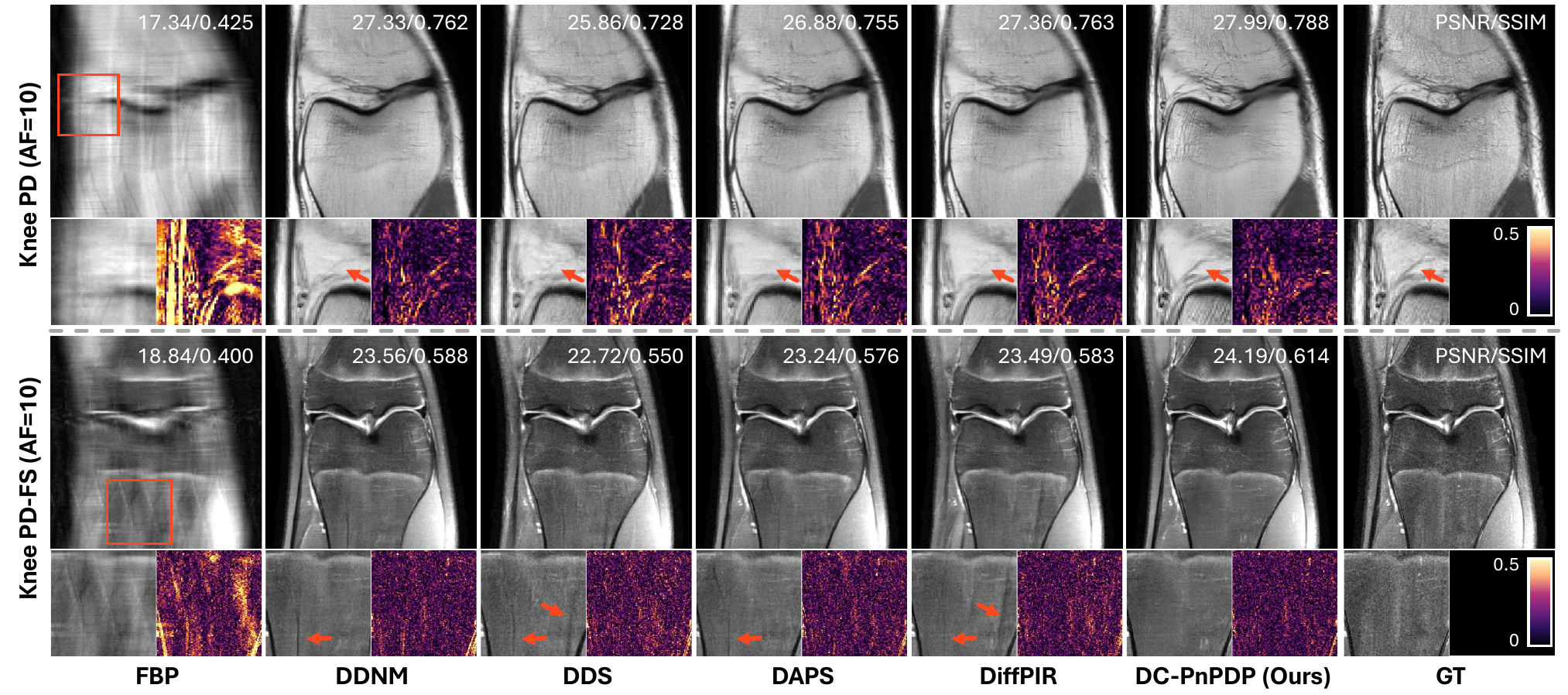}
    \caption{Qualitative results for the controlled 50-NFE fastMRI knee reconstruction setting at AF = 10. The figure shows proton density (PD) and proton density fat-suppressed (PD-FS) scans; red arrows indicate representative improvements and remaining aliasing artifacts.}
    \label{fig:controlled_knee_mri}
\end{figure}

\subsection{Discussion: Modality-Dependent Gain Magnitudes}
\label{app:modality_gain_discussion}
The magnitude of the improvement produced by DC-PnPDP varies across tasks. This variation should not be interpreted as a modality-specific limitation of the method; rather, it is largely determined by two factors: the geometric structure of the forward degradation and the signal-to-noise ratio (SNR) of the reference data used for metric evaluation.

\paragraph{Large Gains on CT.}
The largest gains appear in CT, especially in limited-angle reconstruction. In LACT, an entire angular wedge is unobserved, so the degradation does not merely introduce isotropic blur or additive noise. It creates highly directional streaking artifacts and can shift anatomical boundaries along the missing-angle directions. Memoryless PnP solvers can therefore converge to a biased fixed point that is locally plausible under the diffusion prior but geometrically inconsistent with the measurements. The dual variable in DC-PnPDP acts as an accumulated feedback path that corrects this persistent consensus error, while SH prevents the resulting structured residuals from violating the denoiser's AWGN input assumption. Correcting such geometric misalignment can yield large improvements in pixel-wise metrics, which explains the pronounced PSNR gains observed on LACT and SVCT.

\paragraph{Smaller Metric Gains on fastMRI Knee.}
The fastMRI knee task shows a different regime. The knee data have relatively lower SNR, and pixel-wise metrics such as PSNR can be substantially affected by background noise and acquisition-dependent intensity variability. In this setting, improvements in structural alignment and artifact suppression are still visible, as shown in Figure~\ref{fig:controlled_knee_mri}, but the corresponding PSNR gain is partially masked by noise components that are not directly corrected by the reconstruction algorithm.

\paragraph{Complementary Evidence from fastMRI Brain.}
The fastMRI brain task provides a complementary check of this interpretation. Brain MRI has cleaner anatomical structure and higher effective SNR in our evaluation subset, so coherent undersampling artifacts contribute more strongly to the measured error. Under this condition, the benefit of dual-coupled feedback becomes more visible: at AF=10 the gain increases to +2.24 dB. Thus, the observed gain pattern across CT, knee MRI, and brain MRI is consistent with the proposed mechanism: DC-PnPDP is most beneficial when the dominant error is a structured, physically induced residual that can be accumulated and corrected by the dual feedback loop.

\subsection{Runtime and Memory Profile}
The NFE comparisons in the main text measure how efficiently different methods use calls to the diffusion prior. However, NFE alone does not fully characterize computational cost, because different solvers may attach different data-consistency or inner-optimization procedures to each denoising call. We therefore also profile wall-clock runtime and memory usage to verify that the observed reconstruction gains are not obtained by making each iteration substantially more expensive.

Table~\ref{tab:runtime_profile} reports implementation-level measurements on SVCT-20 using an NVIDIA A100 80GB GPU with batch size 10. These timings should be interpreted as practical references under a shared implementation and hardware setting rather than as universal complexity constants. Compared with DiffPIR, DC-PnPDP keeps the same dominant denoising call and the same data-consistency cost, while adding only the dual update and the FFT-based Spectral Homogenization (SH) step.

\begin{table}[!htbp]
    \centering
    \caption{Runtime and memory profile on SVCT-20 with batch size 10. \textit{Time / Iter.} denotes the total wall-clock time per outer iteration; \textit{Denoise}, \textit{DC}, and \textit{SH} denote the time spent on the diffusion prior step, data-consistency step, and Spectral Homogenization module, respectively. All timings are reported in seconds under the same hardware and batch setting, and \textit{VRAM} denotes peak GPU memory usage.}
    \label{tab:runtime_profile}
    \footnotesize
    \setlength{\tabcolsep}{4pt}
    \begin{tabular}{lccccc}
        \toprule
        \textbf{Method}          & \textbf{Time / Iter. (s)} & \textbf{Denoise (s)} & \textbf{DC (s)} & \textbf{SH (s)} & \textbf{VRAM (MB)} \\
        \midrule
        SITCOM                   & 7.4618                    & 4.4326               & 7.4578          & --              & 38009              \\
        DAPS                     & 0.3190                    & 0.1856               & 0.1333          & --              & 3988               \\
        DiffPIR                  & 0.1951                    & 0.1856               & 0.0095          & --              & 3988               \\
        DC-PnPDP (\textit{Ours}) & 0.1958                    & 0.1856               & 0.0095          & 0.0007          & 3988               \\
        \bottomrule
    \end{tabular}
\end{table}

The breakdown of the data-consistency (DC) cost further clarifies the computational trade-off. DiffPIR and DC-PnPDP solve the linear DC subproblem with a conjugate-gradient (CG) solver, which only requires applications of the forward operator and its adjoint and is therefore highly efficient when the measurement model is explicitly available. DAPS instead performs Langevin dynamics for the DC step, resulting in repeated gradient-based updates under the measurement model and a moderately higher DC cost. SITCOM is substantially more expensive: its DPS-style inner optimization differentiates a measurement-consistency objective through the denoising network, which requires vector-Jacobian products (equivalently, backpropagating through the full denoiser) at each inner step.

This distinction is important for the large-scale medical inverse problems targeted in this work. Volumetric CT and MRI reconstruction often involve large 3D arrays, where repeated denoiser backpropagation can become a major runtime and memory bottleneck. When the forward model is well specified and the adjoint operator can be implemented explicitly, directly solving the resulting linearized DC system with CG is typically much more scalable than DPS-style gradient guidance through the generative network. The measured overhead of DC-PnPDP over DiffPIR is negligible in this setting: the per-iteration runtime changes from 0.1951 s to 0.1958 s, and the SH module itself costs only 0.0007 s. The peak memory is also unchanged. These measurements support the interpretation that the empirical gains of DC-PnPDP primarily come from a better data-consistency feedback mechanism rather than from a larger per-step computational budget.

\subsection{Residual and Inference-Budget Diagnostics}
The PSNR-vs-NFE curves in the main text are intended to measure inference efficiency under a fixed denoiser budget, but they are not by themselves a formal optimization-convergence analysis. To examine the behavior that is most directly tied to the proposed mechanism, we additionally track residual-based quantities: the measurement residual $\|\mathbf{A}\mathbf{x}^{(k)}-\mathbf{y}\|_2$ and the norm of the accumulated dual state. These diagnostics are closer to the physical constraint and to the predicted steady-state bias of memoryless PnP solvers.

\begin{figure}[htbp]
    \centering
    \includegraphics[width=0.8\textwidth]{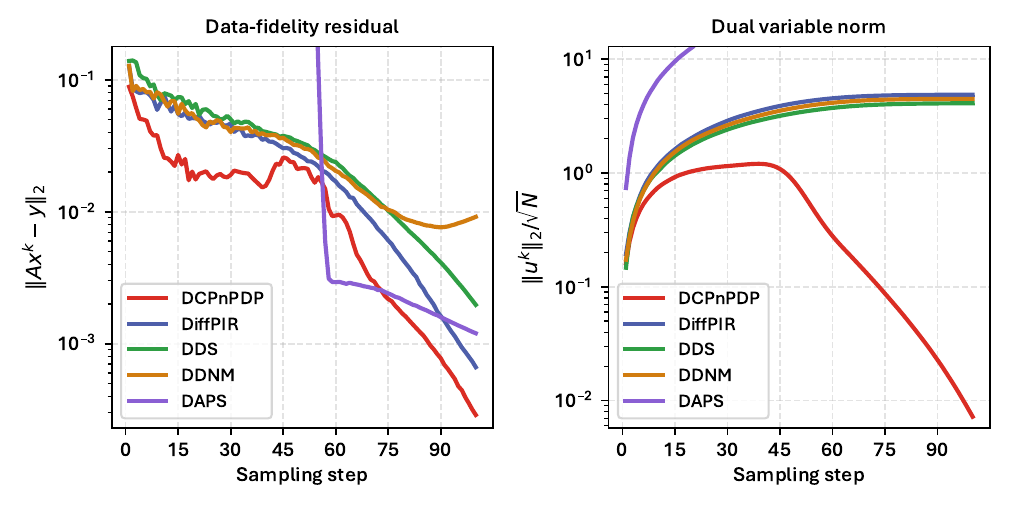}
    \caption{Residual diagnostics on LACT-90. The curves report the data-fidelity residual and the norm of the accumulated dual state, providing optimization-oriented evidence complementary to PSNR-based evaluation.}
    \label{fig:additional_residual_curves}
\end{figure}

Figure~\ref{fig:additional_residual_curves} reports these quantities on LACT-90, a setting where the missing angular range induces strongly directional artifacts and therefore stresses the data-consistency mechanism. DC-PnPDP reduces the measurement residual more effectively than the memoryless baselines. At the same time, the dual-state curve shows that the accumulated mismatch is actively corrected rather than left as a persistent consensus gap. This behavior is consistent with the fixed-point analysis in Appendix~\ref{app:theory}: the dual variable acts as a feedback path that removes the steady-state bias of purely instantaneous updates.

\flushbottom
\subsection{Additional Baseline and Scope Comparisons}
The main experiments focus on zero-shot diffusion-based inverse solvers. To place these results in a broader application context, we also compare with ProCT~\citep{ma2026universal}, a recent supervised incomplete-view CT reconstruction model. This comparison is not intended to replace a full supervised benchmark, since supervised methods depend strongly on their training distribution and task specification. Instead, it tests a practical question: how does a zero-shot dual-coupled diffusion solver compare against a strong task-specific supervised model when both in-domain and OOD settings are considered?

Table~\ref{tab:proct_comparison} reports per-task PSNR on AAPM and DeepLesion. DeepLesion is in-domain for ProCT, while AAPM is out-of-distribution for ProCT; our diffusion prior is not trained specifically for either test set. Across both datasets and across both sparse-view CT (SVCT) and limited-angle CT (LACT), DC-PnPDP obtains the best PSNR in all reported settings. The advantage is particularly clear in the OOD AAPM setting, which is consistent with the zero-shot design goal of combining a general diffusion prior with an explicit physics-based data-consistency loop.

\begin{table}[!htbp]
    \centering
    \caption{Quantitative PSNR comparison against the supervised ProCT baseline on AAPM and DeepLesion. Results are grouped by dataset and CT setting; the best results are highlighted in \textbf{bold} and the second-best are \underline{underlined}.}
    \label{tab:proct_comparison}
    \footnotesize
    \setlength{\tabcolsep}{2.5pt}
    \begin{tabular}{llccccccc}
        \toprule
        \multirow{2}{*}{\textbf{Dataset}} & \multirow{2}{*}{\textbf{Method}} &
        \multicolumn{4}{c}{\textbf{SVCT}} &
        \multicolumn{3}{c}{\textbf{LACT}}                                                                                                                                                                                              \\
        \cmidrule(lr){3-6}\cmidrule(lr){7-9}
                                          &                                  & \textbf{18}         & \textbf{36}         & \textbf{72}         & \textbf{144}        & \textbf{90}         & \textbf{120}        & \textbf{150}        \\
        \midrule
        \multirow{3}{*}{AAPM}
                                          & ProCT                            & 30.2957             & 33.1703             & 35.3471             & 37.2768             & \underline{31.8320} & 34.5075             & 37.0916             \\
                                          & DiffPIR                          & \underline{33.3762} & \underline{38.9284} & \underline{43.0297} & \underline{47.3725} & 28.8041             & \underline{34.5525} & \underline{43.2935} \\
                                          & DC-PnPDP (\textit{Ours})         & \textbf{36.4784}    & \textbf{40.5853}    & \textbf{43.9589}    & \textbf{48.1579}    & \textbf{33.8669}    & \textbf{39.2974}    & \textbf{45.0820}    \\
        \midrule
        \multirow{3}{*}{DeepLesion}
                                          & ProCT                            & 28.8959             & 30.9438             & 32.6598             & 33.8446             & \underline{31.5145} & 33.0290             & 34.0849             \\
                                          & DiffPIR                          & \underline{30.4603} & \underline{35.8175} & \underline{40.5997} & \underline{44.7531} & 27.8810             & \underline{34.7194} & \underline{43.1713} \\
                                          & DC-PnPDP (\textit{Ours})         & \textbf{33.7165}    & \textbf{38.0110}    & \textbf{41.9861}    & \textbf{45.7226}    & \textbf{33.0598}    & \textbf{39.3003}    & \textbf{45.4183}    \\
        \bottomrule
    \end{tabular}
\end{table}

The qualitative comparisons in Figures~\ref{fig:proct_visual_lact} and~\ref{fig:proct_visual_svct} provide complementary evidence. In LACT, where an entire angular wedge is missing, memoryless or purely feed-forward reconstructions can exhibit geometrically structured errors rather than only local blur. DC-PnPDP reduces these structured residuals in the zoomed regions and error maps, suggesting that the dual feedback is most beneficial when the inverse problem induces coherent, physically structured artifacts.

\begin{figure}[htbp]
    \centering
    \includegraphics[width=0.86\textwidth]{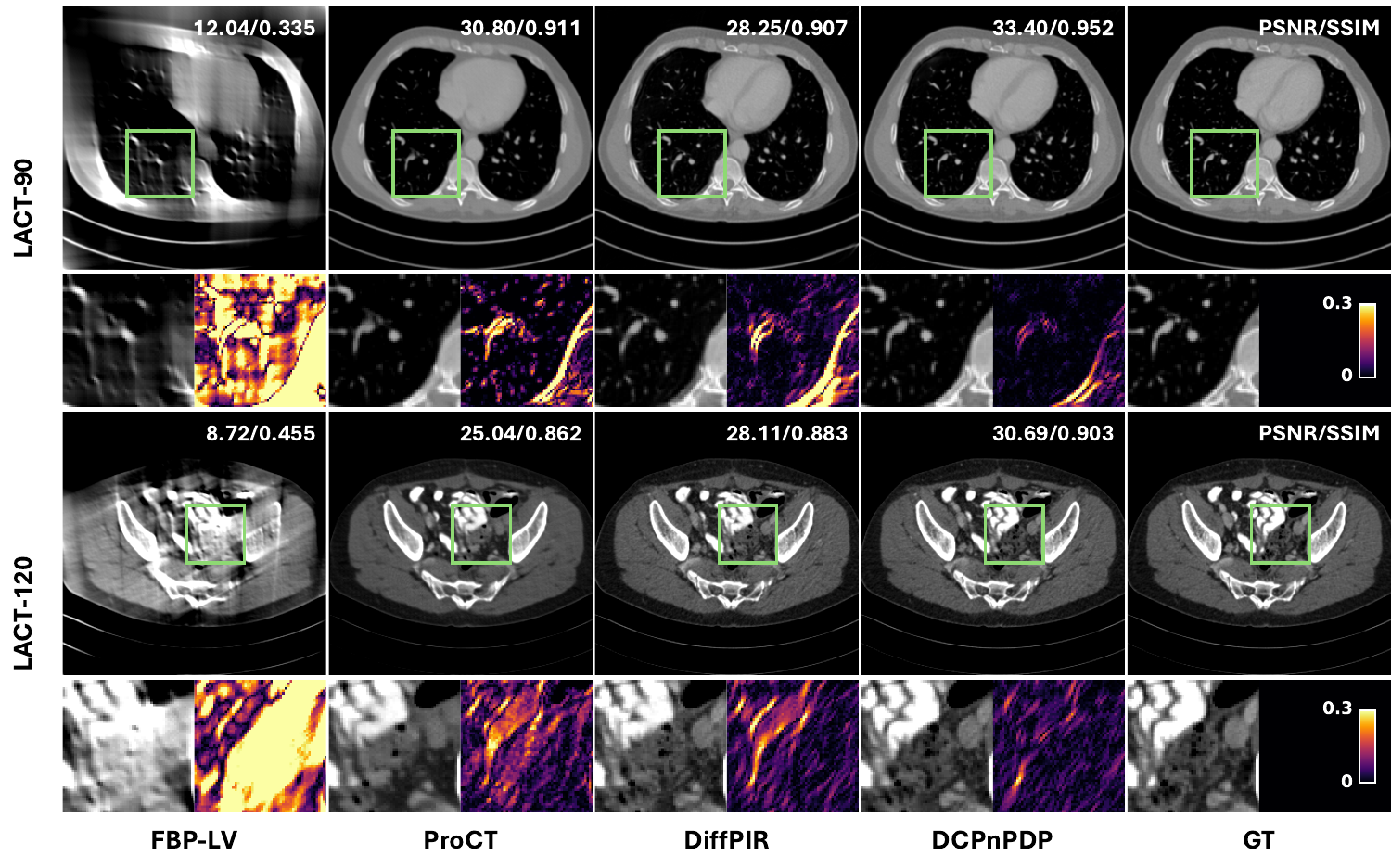}
    \caption{Additional qualitative comparison with ProCT on LACT-90 and LACT-120. The figure includes reconstruction results, zoomed-in regions, and error maps, with FBP-LV and ground truth shown for reference.}
    \label{fig:proct_visual_lact}
\end{figure}

\begin{figure}[htbp]
    \centering
    \includegraphics[width=0.86\textwidth]{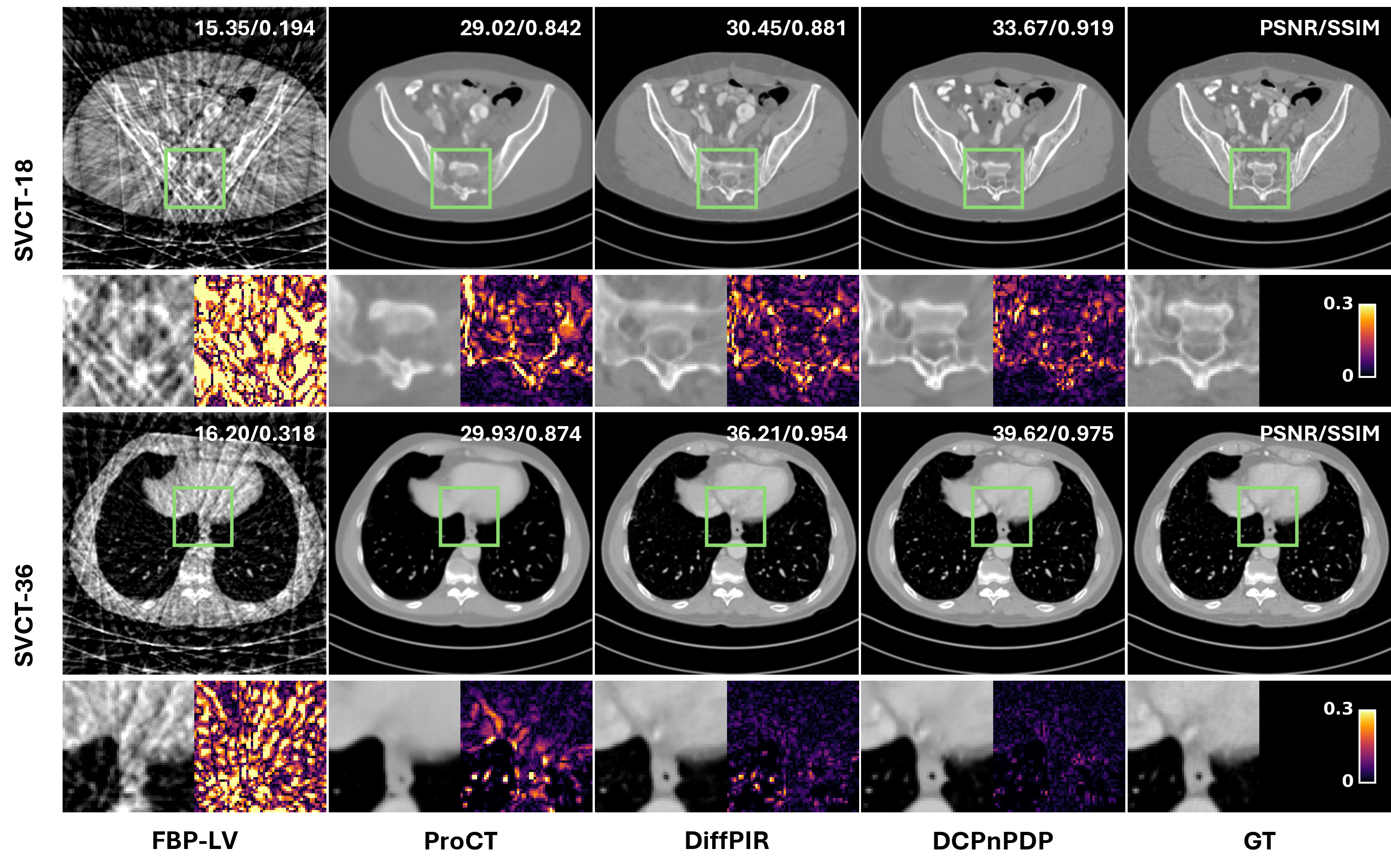}
    \caption{Additional qualitative comparison with ProCT on SVCT-18 and SVCT-36. The figure includes reconstruction results, zoomed-in regions, and error maps, with FBP-LV and ground truth shown for reference.}
    \label{fig:proct_visual_svct}
\end{figure}
\clearpage

\subsection{Natural-Image Inverse Problems}
Although the paper is motivated by CT and MRI, the memoryless-bias issue is not specific to medical imaging. It arises whenever a PnP solver repeatedly alternates a learned prior with a data-consistency step without retaining the accumulated constraint mismatch. To test whether the proposed dual-coupled mechanism transfers beyond the medical setting, we evaluate standard natural-image linear inverse problems: Gaussian deblurring, motion deblurring, and $4{\times}$ super-resolution.

Table~\ref{tab:natural_image_results} and Figure~\ref{fig:natural_image_results} compare DC-PnPDP with DAPS and DiffPIR on these tasks. DC-PnPDP achieves the best PSNR and SSIM on all three inverse problems, indicating that the dual-coupled update is not tied to a particular CT or MRI forward model. LPIPS is reported for compatibility with natural-image restoration benchmarks, but it should be interpreted together with the perception-distortion trade-off: sampling-heavy methods may sometimes produce more perceptually favored textures, while a stricter data-consistency mechanism tends to favor faithful recovery under the measurement model.

\begin{table}[!htbp]
    \centering
    \caption{Quantitative natural-image restoration results on Gaussian blur, motion blur, and $4{\times}$ super-resolution. Results are reported as mean $\pm$ standard deviation; the best results are highlighted in \textbf{bold} and the second-best are \underline{underlined}.}
    \label{tab:natural_image_results}
    \footnotesize
    \setlength{\tabcolsep}{5pt}
    \begin{tabular}{llccc}
        \toprule
        \textbf{Task}                   & \textbf{Method} & \textbf{PSNR} $\uparrow$       & \textbf{SSIM} $\uparrow$      & \textbf{LPIPS} $\downarrow$   \\
        \midrule
        \multirow{3}{*}{Gaussian Blur}  & DAPS            & 29.609 $\pm$ 1.637             & 0.786 $\pm$ 0.031             & 0.180 $\pm$ 0.030             \\
                                        & DiffPIR         & \underline{30.506 $\pm$ 1.940} & \underline{0.858 $\pm$ 0.038} & \textbf{0.145 $\pm$ 0.029}    \\
                                        & DC-PnPDP        & \textbf{30.741 $\pm$ 2.011}    & \textbf{0.861 $\pm$ 0.037}    & \underline{0.165 $\pm$ 0.031} \\
        \midrule
        \multirow{3}{*}{Motion Blur}    & DAPS            & \underline{31.653 $\pm$ 1.337} & \underline{0.836 $\pm$ 0.021} & 0.138 $\pm$ 0.027             \\
                                        & DiffPIR         & 31.088 $\pm$ 1.451             & 0.831 $\pm$ 0.026             & \underline{0.129 $\pm$ 0.023} \\
                                        & DC-PnPDP        & \textbf{33.305 $\pm$ 1.622}    & \textbf{0.904 $\pm$ 0.023}    & \textbf{0.116 $\pm$ 0.027}    \\
        \midrule
        \multirow{3}{*}{$4{\times}$ SR} & DAPS            & 29.363 $\pm$ 1.492             & 0.782 $\pm$ 0.029             & 0.193 $\pm$ 0.033             \\
                                        & DiffPIR         & \underline{30.330 $\pm$ 1.809} & \underline{0.858 $\pm$ 0.035} & \textbf{0.154 $\pm$ 0.028}    \\
                                        & DC-PnPDP        & \textbf{30.807 $\pm$ 1.862}    & \textbf{0.866 $\pm$ 0.033}    & \underline{0.158 $\pm$ 0.029} \\
        \bottomrule
    \end{tabular}
\end{table}

\begin{figure}[!htbp]
    \centering
    \includegraphics[width=1\textwidth]{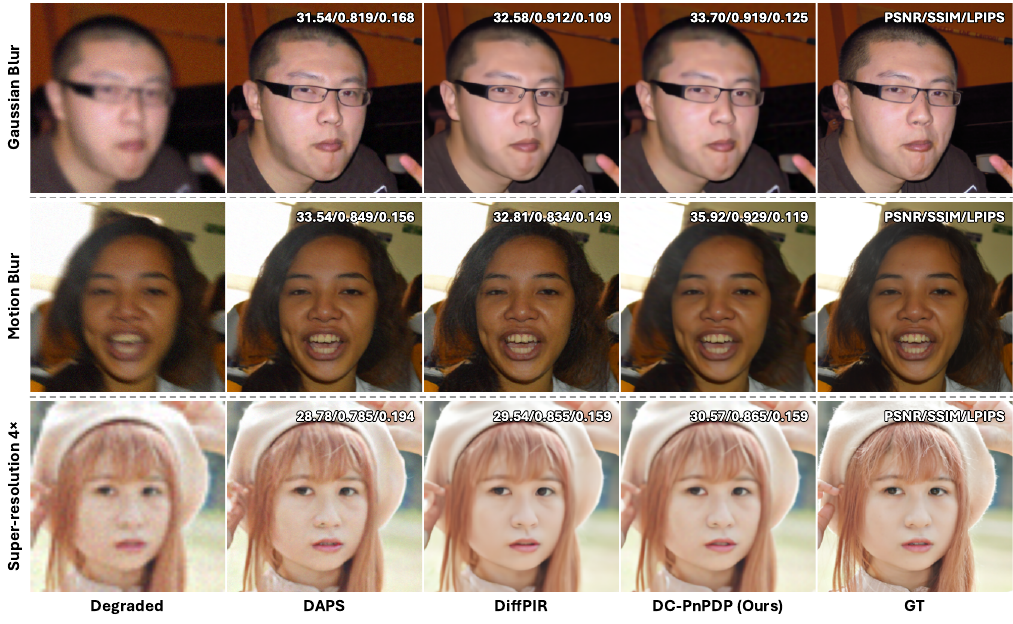}
    \caption{Additional qualitative examples for natural-image inverse problems. The figure compares restoration results under Gaussian blur, motion blur, and $4{\times}$ super-resolution.}
    \label{fig:natural_image_results}
\end{figure}
\clearpage

\paragraph{Applying SH to Other PnP Solvers.}
The SH module is designed to address a statistical compatibility problem: after a data-consistency step, the residual fed to the diffusion prior can be spectrally colored rather than AWGN-like. This issue can occur in other PnP solvers as well, even when they do not maintain an explicit ADMM dual variable. As a lightweight diagnostic, we inserted SH into DDNM by replacing the standard DDIM re-noising step with the proposed spectral correction. On SVCT-20, this modification improves PSNR from 34.82 to 35.07.

This gain is intentionally reported as a scope check rather than as a new method variant. SH is most useful when the solver creates severe structured residuals, which is precisely what can happen in DC-PnPDP because the dual state accumulates historical constraint violations. Methods without an explicit memory variable typically produce a milder distribution shift, so the isolated benefit of SH is naturally smaller. This observation supports the design choice of using SH together with, rather than instead of, dual coupling.

\section{Justification of Spectral Residual Estimation via Bootstrapping}
\label{app:spectral_justification}

In Section \ref{sec:sh}, we approximate the true solver-induced residual using a bootstrap strategy, defined as $\mathbf{r}^{(k+1)} = \mathbf{v}^{(k+1)} - \mathbf{z}^{(k)}$. A natural concern is whether the estimation error in the pilot estimate $\mathbf{z}^{(k)}$ (especially in early iterations) compromises the accuracy of the spectral diagnosis. Here, we provide a detailed analysis of why this approximation holds robustly in the context of inverse problems.

\subsection{Error Decomposition}
Let $\mathbf{x}_{\text{GT}}$ denote the ground truth image. The \textit{true} solver residual containing the physics-based artifacts is $\mathbf{r}_{\text{true}} = \mathbf{v}^{(k+1)} - \mathbf{x}_{\text{GT}}$.
Our empirical estimate $\mathbf{r}^{(k+1)}$ can be decomposed as:
\begin{align}
    \mathbf{r}^{(k+1)} & = \mathbf{v}^{(k+1)} - \mathbf{z}^{(k)} \nonumber                                                                                                  \\
                       & = (\mathbf{v}^{(k+1)} - \mathbf{x}_{\text{GT}}) + (\mathbf{x}_{\text{GT}} - \mathbf{z}^{(k)}) \nonumber                                            \\
                       & = \underbrace{\mathbf{r}_{\text{true}}}_{\text{Structured Artifacts}} + \underbrace{\boldsymbol{\epsilon}_{\text{est}}}_{\text{Estimation Error}},
\end{align}
where $\boldsymbol{\epsilon}_{\text{est}} = \mathbf{x}_{\text{GT}} - \mathbf{z}^{(k)}$ represents the denoising error of the previous step.
The validity of our Spectral Homogenization relies on the condition that the spectral signature of $\mathbf{r}_{\text{true}}$ is distinguishable from that of $\boldsymbol{\epsilon}_{\text{est}}$.

\subsection{Spectral Separability Analysis}
We analyze the properties of these two terms in the frequency domain:

\textbf{1. Structured Artifacts ($\mathbf{r}_{\text{true}}$): High Spectral Sparsity.}
In ill-posed inverse problems, the artifacts in $\mathbf{v}^{(k+1)}$ (derived from $\mathbf{x}^{(k+1)} + \mathbf{u}^{(k)}$) are governed by the geometry of the forward operator $\mathbf{A}$.
\begin{itemize}
    \item For \textbf{Sparse-View CT}, artifacts manifest as streaking patterns radiating from high-contrast regions. These correspond to energy concentration in a small set of angular frequency bands (a \textit{sparse} support in Fourier space).
    \item For \textbf{Accelerated MRI}, artifacts manifest as periodic aliasing (Cartesian undersampling). This results in sharp, coherent peaks in the spectrum.
\end{itemize}
Mathematically, the PSD $S_{\mathbf{r}_{\text{true}}}(\boldsymbol{\omega})$ exhibits high kurtosis, with energy concentrated in narrow bandwidths $\Omega_{\text{art}}$:
\begin{equation}
    S_{\mathbf{r}_{\text{true}}}(\boldsymbol{\omega}) \gg 0, \quad \forall \boldsymbol{\omega} \in \Omega_{\text{art}}.
\end{equation}

\textbf{2. Estimation Error ($\boldsymbol{\epsilon}_{\text{est}}$): Spectral Diffusion.}
The term $\boldsymbol{\epsilon}_{\text{est}}$ represents the residual noise that the diffusion denoiser failed to remove or the details it failed to recover.
Modern diffusion models (like DDPM/EDM) are trained to predict Gaussian noise. Consequently, their prediction errors tend to be incoherent and broadband (spectrally flat or slowly varying), lacking the strong geometric directionality of the physics artifacts.
Thus, $S_{\boldsymbol{\epsilon}_{\text{est}}}(\boldsymbol{\omega})$ acts as a low-level, broadband background floor.

\subsection{The Dominance Condition}
Spectral Homogenization effectively detects artifacts when the local Signal-to-Noise Ratio (SNR) in the frequency domain is high. Even if the total $L_2$ error of the estimate is large (i.e., $\|\boldsymbol{\epsilon}_{\text{est}}\|_2^2 \approx \|\mathbf{r}_{\text{true}}\|_2^2$), the peak energy density differs by orders of magnitude due to the energy compaction property of the artifacts:
\begin{equation}
    \max_{\boldsymbol{\omega}} S_{\mathbf{r}_{\text{true}}}(\boldsymbol{\omega}) \gg \max_{\boldsymbol{\omega}} S_{\boldsymbol{\epsilon}_{\text{est}}}(\boldsymbol{\omega}).
\end{equation}
By computing the PSD, we effectively project the signal onto a basis where the artifacts are sparse and the error is dense. This allows the smoothing kernel $\mathcal{K}_\delta$ (Eq. \ref{eq:psd_est}) to capture the envelope of the artifacts $S_{\mathbf{r}_{\text{true}}}$ while averaging out the incoherent fluctuations of $S_{\boldsymbol{\epsilon}_{\text{est}}}$. Therefore, the resulting mask $\Delta S(\boldsymbol{\omega})$ correctly targets the artifacts without being significantly perturbed by the imperfections of $\mathbf{z}^{(k)}$.

\section{Fixed-Point Characterization and Bias Analysis}
\label{sec:convergence_analysis}

While establishing global convergence guarantees for non-convex PnP algorithms with stochastic diffusion priors remains an open theoretical challenge, we can rigorously analyze the properties of the algorithm's \textit{fixed points}. This analysis reveals why the dual coupling is essential for unbiased reconstruction.

Let $(\mathbf{x}^*, \mathbf{z}^*, \mathbf{u}^*)$ denote a fixed point of the iteration, assuming the spectral homogenization acts as a zero-mean perturbation in expectation at steady state.

\begin{theorem}[Optimality of Dual-Coupled Fixed Points]
    \label{thm:fixed_point}
    Assume the denoiser $\mathcal{D}_\sigma$ acts as a proximal operator for some implicit regularizer $\phi(\cdot)$, i.e., $\mathcal{D}_\sigma(\mathbf{v}) = \text{prox}_{\gamma \phi}(\mathbf{v})$.
    1. The fixed point $(\mathbf{x}^*, \mathbf{z}^*, \mathbf{u}^*)$ of the proposed Dual-Coupled PnP ADMM satisfies the first-order optimality condition of the original problem:
    \begin{equation}
        \mathbf{0} \in \nabla f(\mathbf{x}^*) + \partial \phi(\mathbf{x}^*).
    \end{equation}
    2. Conversely, a PnP scheme that discards the dual variable (setting $\mathbf{u} \equiv \mathbf{0}$) converges to a biased solution $\tilde{\mathbf{x}}$ satisfying:
    \begin{equation}
        \mathbf{0} \in \nabla f(\tilde{\mathbf{x}}) + \partial \phi(\tilde{\mathbf{x}}) + \underbrace{(\tilde{\mathbf{x}} - \mathcal{D}_\sigma(\tilde{\mathbf{x}}))}_{\text{Systematic Bias}}.
    \end{equation}
\end{theorem}

\begin{proof}
    \textit{Part 1 (Proposed Method):} At a fixed point, $\mathbf{u}^{k+1} = \mathbf{u}^k$, which implies $\mathbf{x}^* - \mathbf{z}^* = \mathbf{0}$ via Eq. \eqref{eq:admm_u}, ensuring consensus $\mathbf{x}^* = \mathbf{z}^*$.
    The x-update \eqref{eq:admm_x} implies $\mathbf{0} = \nabla f(\mathbf{x}^*) + \rho(\mathbf{x}^* - \mathbf{z}^* + \mathbf{u}^*) \implies \nabla f(\mathbf{x}^*) = -\rho \mathbf{u}^*$.
    The z-update (prior step) implies $\mathbf{z}^* = \text{prox}_{\gamma \phi}(\mathbf{x}^* + \mathbf{u}^*)$. By definition of the proximal operator, this yields $\mathbf{x}^* + \mathbf{u}^* - \mathbf{z}^* \in \gamma \partial \phi(\mathbf{z}^*)$. Substituting $\mathbf{x}^*=\mathbf{z}^*$, we get $\mathbf{u}^* \in \gamma \partial \phi(\mathbf{x}^*)$.
    Combining these, $\nabla f(\mathbf{x}^*) = -\rho \mathbf{u}^* \in -\rho \gamma \partial \phi(\mathbf{x}^*)$. Setting parameters consistently recovers $\mathbf{0} \in \nabla f + \partial \phi$.

    \textit{Part 2 (Loose Coupling / No Dual):} If $\mathbf{u}$ is omitted, the updates effectively become alternating minimization without memory. The x-update becomes $\nabla f(\tilde{\mathbf{x}}) + \rho(\tilde{\mathbf{x}} - \tilde{\mathbf{z}}) = 0$. The z-update becomes $\tilde{\mathbf{z}} = \mathcal{D}_\sigma(\tilde{\mathbf{x}})$.
    Substituting $\tilde{\mathbf{z}}$, we get $\nabla f(\tilde{\mathbf{x}}) + \rho(\tilde{\mathbf{x}} - \mathcal{D}_\sigma(\tilde{\mathbf{x}})) = 0$.
    The term $\rho(\tilde{\mathbf{x}} - \mathcal{D}_\sigma(\tilde{\mathbf{x}}))$ represents the gradient of the prior $\nabla \phi(\tilde{\mathbf{x}})$ \textit{only if} $\tilde{\mathbf{x}}$ is already a perfect denoised image. However, in practice, $\tilde{\mathbf{x}}$ contains residual noise. The discrepancy leads to a non-vanishing gradient term, creating a permanent fixed-point bias where data consistency fights against the denoiser without the buffering capacity of the dual variable to absorb the residual.
\end{proof}

\section{Theoretical Analysis of Fixed-Point Properties}
\label{app:theory}

In this section, we provide a rigorous analysis of the fixed-point characterization of the proposed Dual-Coupled PnP framework. While establishing global convergence sequences for non-convex, stochastic PnP algorithms remains an open problem in the field, analyzing the \textit{fixed points} (i.e., the steady states) of the algorithm provides critical insights into the correctness of the solution.

Specifically, we prove that the explicit inclusion of the dual variable enables the framework to converge to a stationary point of the original variational problem, whereas standard PnP-HQS methods (which omit the dual variable) suffer from a systematic bias.

\subsection{Preliminaries and Assumptions}

We consider the composite optimization problem:
\begin{equation}
    \min_{\mathbf{x}} \quad f(\mathbf{x}) + \lambda \phi(\mathbf{x}),
    \label{eq:app_prob}
\end{equation}
where $f(\mathbf{x}) = \frac{1}{2}\|\mathbf{A}\mathbf{x} - \mathbf{y}\|_\mathbf{W}^2$ is the data-fidelity term, and $\phi(\mathbf{x})$ is an implicit regularizer induced by the image prior.

To connect the denoiser to the optimization objective, we adopt the standard assumption in PnP literature regarding the relationship between the denoiser and the proximal operator.

\begin{assumption}[Denoiser as Proximal Operator]
    \label{asm:prox}
    We assume that the pretrained denoiser $\mathcal{D}_\sigma$ acts locally as the proximal operator of the regularizer $\phi$ with parameter $\gamma > 0$:
    \begin{equation}
        \mathcal{D}_\sigma(\mathbf{v}) \approx \text{prox}_{\gamma \phi}(\mathbf{v}) \triangleq \arg\min_{\mathbf{z}} \left( \frac{1}{2}\|\mathbf{z} - \mathbf{v}\|_2^2 + \gamma \phi(\mathbf{z}) \right).
    \end{equation}
    For diffusion models, this is motivated by Tweedie's formula, which links the MMSE denoiser to the score function $\nabla \log p(\mathbf{x})$.
\end{assumption}

\subsection{Theorem 1: Optimality of Dual-Coupled Fixed Points}

\begin{theorem}
    Let $(\mathbf{x}^*, \mathbf{z}^*, \mathbf{u}^*)$ be a fixed point of the deterministic backbone of the proposed Dual-Coupled PnP ADMM iterations. Under Assumption \ref{asm:prox}, $\mathbf{x}^*$ satisfies the first-order optimality condition of the original problem \eqref{eq:app_prob} with effective regularization weight $\lambda = \rho \gamma$.
\end{theorem}

\begin{proof}
    The deterministic updates of our ADMM formulation are:
    \begin{align}
        \mathbf{x}^{k+1} & := \arg\min_{\mathbf{x}} f(\mathbf{x}) + \frac{\rho}{2}\|\mathbf{x} - \mathbf{z}^k + \mathbf{u}^k\|_2^2 \label{eq:app_x_update} \\
        \mathbf{z}^{k+1} & := \mathcal{D}_\sigma(\mathbf{x}^{k+1} + \mathbf{u}^k) \label{eq:app_z_update}                                                  \\
        \mathbf{u}^{k+1} & := \mathbf{u}^k + \mathbf{x}^{k+1} - \mathbf{z}^{k+1} \label{eq:app_u_update}
    \end{align}
    At a fixed point $(\mathbf{x}^*, \mathbf{z}^*, \mathbf{u}^*)$, we have $\mathbf{x}^{k+1} = \mathbf{x}^k = \mathbf{x}^*$, etc.

    \textbf{Step 1: Consensus.}
    From the dual update \eqref{eq:app_u_update}, stationarity implies:
    \begin{equation}
        \mathbf{u}^* = \mathbf{u}^* + \mathbf{x}^* - \mathbf{z}^* \implies \mathbf{x}^* = \mathbf{z}^*.
        \label{eq:consensus}
    \end{equation}
    Thus, the primal and auxiliary variables reach perfect consensus.

    \textbf{Step 2: Data Consistency Condition.}
    From the x-update \eqref{eq:app_x_update}, the first-order optimality condition is:
    \begin{equation}
        \nabla f(\mathbf{x}^*) + \rho(\mathbf{x}^* - \mathbf{z}^* + \mathbf{u}^*) = \mathbf{0}.
    \end{equation}
    Using the consensus property $\mathbf{x}^* = \mathbf{z}^*$, this simplifies to:
    \begin{equation}
        \nabla f(\mathbf{x}^*) + \rho \mathbf{u}^* = \mathbf{0} \implies \nabla f(\mathbf{x}^*) = -\rho \mathbf{u}^*.
        \label{eq:grad_f}
    \end{equation}
    This physically implies that the dual variable $\mathbf{u}^*$ perfectly balances the gradient of the data-fidelity term.

    \textbf{Step 3: Prior Condition.}
    From the z-update \eqref{eq:app_z_update} and Assumption \ref{asm:prox}:
    \begin{equation}
        \mathbf{z}^* = \text{prox}_{\gamma \phi}(\mathbf{x}^* + \mathbf{u}^*).
    \end{equation}
    By the definition of the proximal operator, $\mathbf{z}^*$ must satisfy:
    \begin{equation}
        \mathbf{0} \in (\mathbf{z}^* - (\mathbf{x}^* + \mathbf{u}^*)) + \gamma \partial \phi(\mathbf{z}^*).
    \end{equation}
    Rearranging terms and utilizing $\mathbf{x}^* = \mathbf{z}^*$:
    \begin{equation}
        \mathbf{u}^* \in \gamma \partial \phi(\mathbf{x}^*).
        \label{eq:subgrad_phi}
    \end{equation}

    \textbf{Step 4: Global Optimality.}
    Combining \eqref{eq:grad_f} and \eqref{eq:subgrad_phi}, we substitute $\mathbf{u}^*$:
    \begin{equation}
        \nabla f(\mathbf{x}^*) \in -\rho (\gamma \partial \phi(\mathbf{x}^*)) \implies \mathbf{0} \in \nabla f(\mathbf{x}^*) + \rho \gamma \partial \phi(\mathbf{x}^*).
    \end{equation}
    Letting $\lambda = \rho \gamma$, this is exactly the optimality condition for $\min f(\mathbf{x}) + \lambda \phi(\mathbf{x})$.
\end{proof}

\subsection{Theorem 2: Systematic Bias in Loosely Coupled PnP}

\begin{theorem}
    Consider a ``Loose PnP'' scheme that omits the dual variable (equivalent to setting $\mathbf{u}^k \equiv \mathbf{0}$). Let $\tilde{\mathbf{x}}$ be a fixed point of this scheme. Then $\tilde{\mathbf{x}}$ satisfies a biased optimality condition:
    \begin{equation}
        \mathbf{0} \in \nabla f(\tilde{\mathbf{x}}) + \frac{\rho}{\sigma^2} \mathcal{E}_{bias}(\tilde{\mathbf{x}}),
    \end{equation}
    where the bias term depends on the residual noise of the denoiser output.
\end{theorem}

\begin{proof}
    Without the dual variable, the updates simplify to:
    \begin{align}
        \tilde{\mathbf{x}} & := \arg\min_{\mathbf{x}} f(\mathbf{x}) + \frac{\rho}{2}\|\mathbf{x} - \tilde{\mathbf{z}}\|_2^2 \\
        \tilde{\mathbf{z}} & := \mathcal{D}_\sigma(\tilde{\mathbf{x}})
    \end{align}
    The x-update implies:
    \begin{equation}
        \nabla f(\tilde{\mathbf{x}}) + \rho(\tilde{\mathbf{x}} - \tilde{\mathbf{z}}) = \mathbf{0}.
    \end{equation}
    Substituting $\tilde{\mathbf{z}} = \mathcal{D}_\sigma(\tilde{\mathbf{x}})$:
    \begin{equation}
        \nabla f(\tilde{\mathbf{x}}) + \rho(\tilde{\mathbf{x}} - \mathcal{D}_\sigma(\tilde{\mathbf{x}})) = \mathbf{0}.
        \label{eq:loose_fixed}
    \end{equation}
    Using Tweedie's formula approximation $\mathcal{D}_\sigma(\mathbf{x}) \approx \mathbf{x} + \sigma^2 \nabla \log p(\mathbf{x})$, the term in the parenthesis becomes:
    \begin{equation}
        \tilde{\mathbf{x}} - \mathcal{D}_\sigma(\tilde{\mathbf{x}}) \approx -\sigma^2 \nabla \log p(\tilde{\mathbf{x}}).
    \end{equation}
    Superficially, substituting this into \eqref{eq:loose_fixed} yields $\nabla f(\tilde{\mathbf{x}}) - \rho \sigma^2 \nabla \log p(\tilde{\mathbf{x}}) \approx \mathbf{0}$, which resembles the MAP condition.

    \textbf{However, a critical contradiction arises:} The diffusion denoiser $\mathcal{D}_\sigma(\cdot)$ is trained to estimate $\mathbf{x}_{clean}$ from a \textit{noisy} input $\mathbf{x}_{clean} + \mathbf{n}$. In the loose PnP scheme, the input to the denoiser is $\tilde{\mathbf{x}}$, which is the current \textit{clean} estimate of the image.
    \begin{itemize}
        \item If $\tilde{\mathbf{x}}$ is truly clean and noise-free, an ideal denoiser acts as an identity map: $\mathcal{D}_\sigma(\tilde{\mathbf{x}}) \approx \tilde{\mathbf{x}}$. In this case, \eqref{eq:loose_fixed} becomes $\nabla f(\tilde{\mathbf{x}}) \approx \mathbf{0}$, meaning the prior is effectively ignored, and the solution collapses to the Maximum Likelihood Estimate (overfitting to noisy measurements).
        \item If $\mathcal{D}_\sigma$ is not an identity map on clean data (which is true for real networks), it introduces a ``hallucination error'' or ``texture smoothing'' bias $\mathcal{E}_{bias} = \tilde{\mathbf{x}} - \mathcal{D}_\sigma(\tilde{\mathbf{x}})$. This term acts as a phantom force that does not correspond to the true score function of the prior distribution.
    \end{itemize}
    Consequently, $\tilde{\mathbf{x}}$ is not a stationary point of $f(\mathbf{x}) + \lambda \phi(\mathbf{x})$, but rather a compromise point dictated by the denoiser's behavior on out-of-distribution (clean) inputs.
\end{proof}

\subsection{Physical Interpretation: The Role of Integral Action}

The contrast between Theorem 1 and Theorem 2 can be understood through a control theory analogy.
\begin{itemize}
    \item \textit{Current PnPDP (No Dual)} acts as a \textbf{Proportional (P) Controller}. It tries to reduce the error between the data-consistency term and the prior term simply by pulling them together ($\rho(\mathbf{x} - \mathbf{z})$). Like any P-controller, it suffers from a ``steady-state error'' (bias) because a non-zero force (gradient) is required to maintain the balance between $f$ and $\phi$.
    \item \textit{Dual-Coupled PnPDP (Our Method)} acts as a \textbf{Proportional-Integral (PI) Controller}. The dual variable $\mathbf{u}$ integrates the discrepancy $(\mathbf{x} - \mathbf{z})$ over time. This integral action builds up a force that exactly cancels the gradient $\nabla f(\mathbf{x}^*)$ at the optimal solution, allowing the error $(\mathbf{x} - \mathbf{z})$ to settle to exactly zero. This ensures that the final result is physically consistent with both the measurements and the prior.
\end{itemize}

\section{Additional Visual Results}

\begin{figure}[!htbp]
    \centering
    \includegraphics[width=0.88\columnwidth]{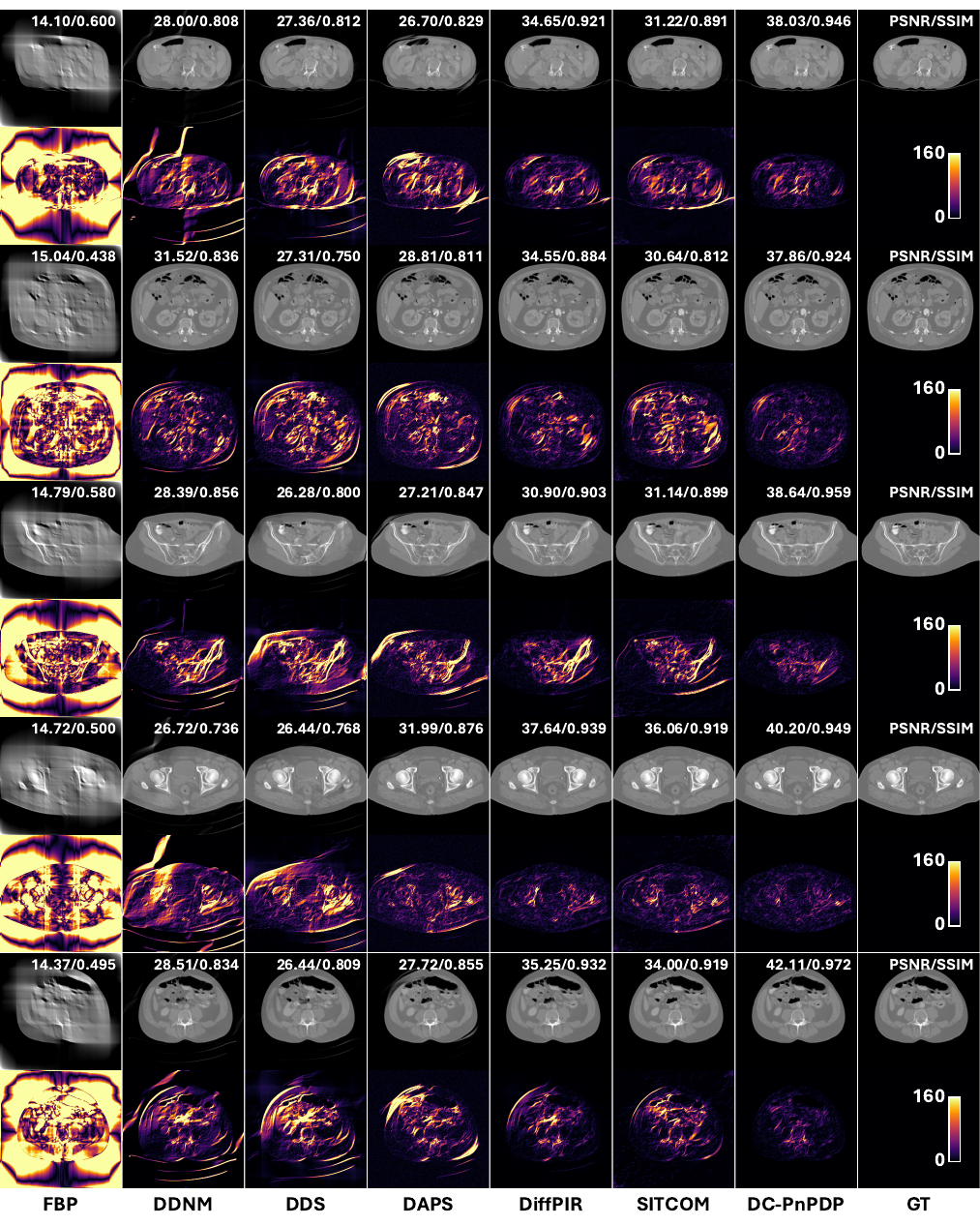}
    \caption{Additional qualitative results on limited-angle CT (LACT) reconstruction with the angular range $[0,90]^\circ$. The visualization window is set to $[-800,800]$ HU.}
    \label{fig:More_Vis_5_LACT_seed_15194776891779460304_chest}
\end{figure}

\begin{figure}[!htbp]
    \centering
    \includegraphics[width=0.88\columnwidth]{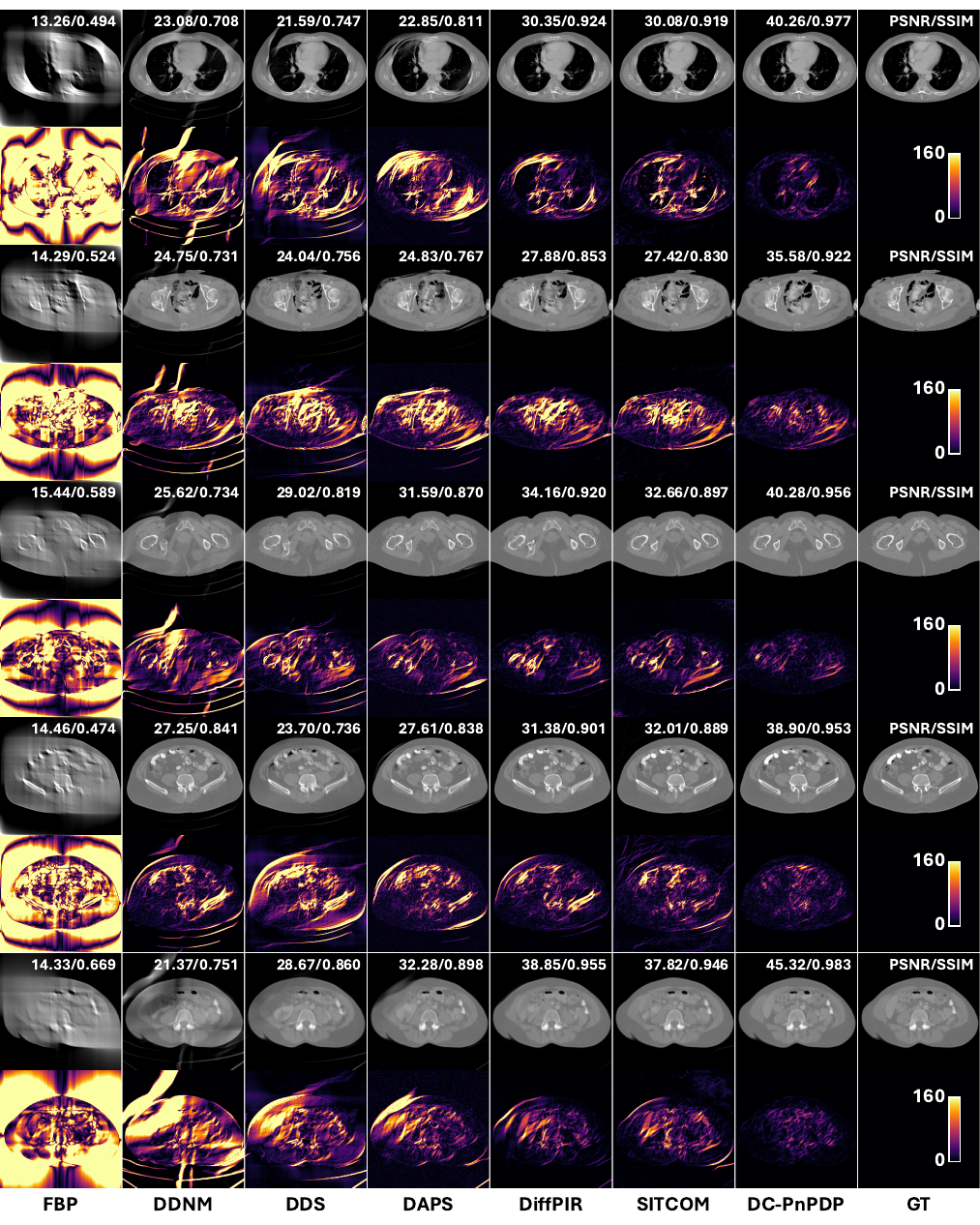}
    \caption{Additional qualitative results on limited-angle CT (LACT) reconstruction with the angular range $[0,90]^\circ$. The visualization window is set to $[-800,800]$ HU.}
    \label{fig:More_Vis_5_LACT_seed2026_other_f150_250}
\end{figure}

\begin{figure}[!htbp]
    \centering
    \includegraphics[width=0.88\columnwidth]{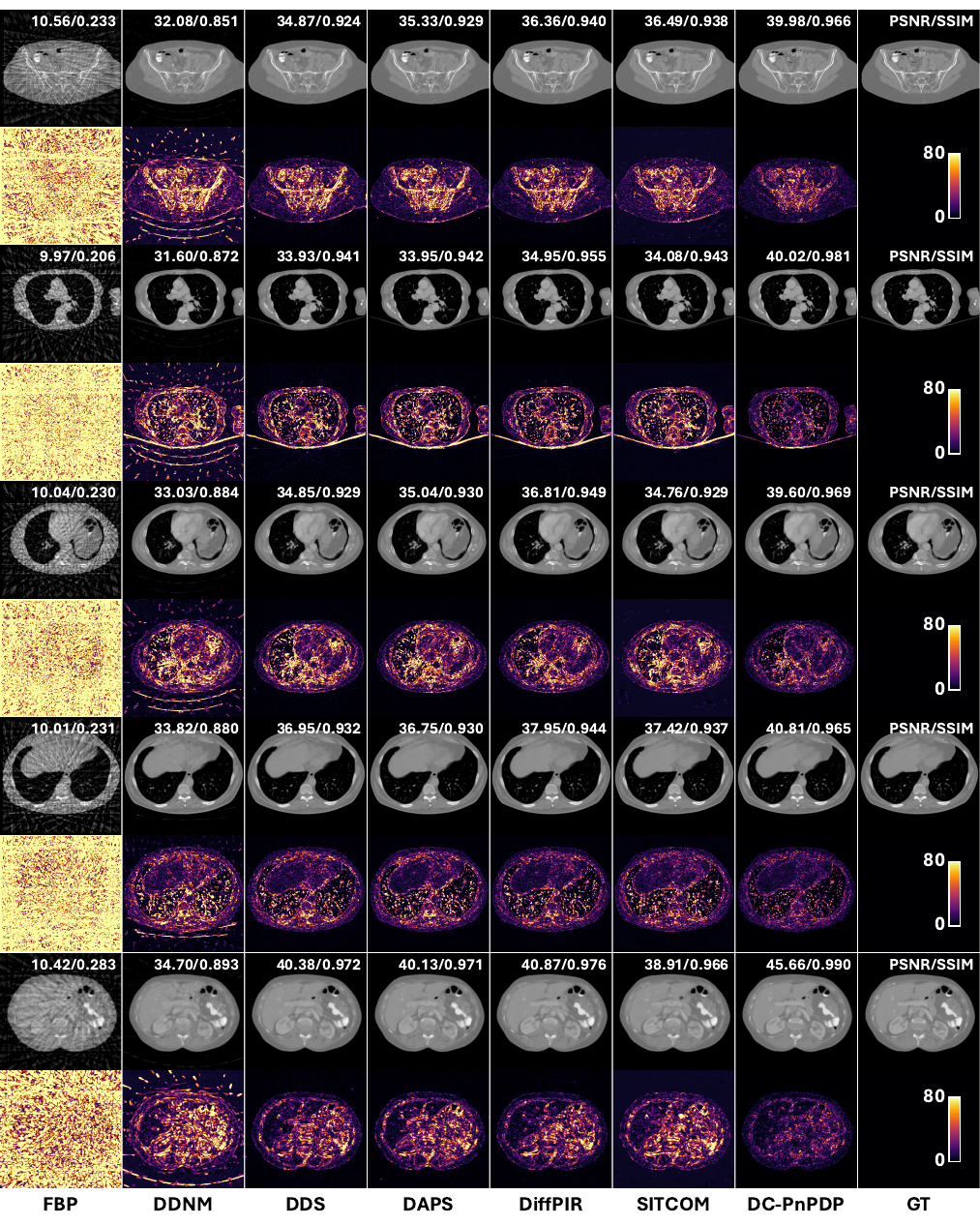}
    \caption{Additional qualitative results on sparse-view CT (SVCT) reconstruction with 20 views. The visualization window is set to $[-800,800]$ HU.}
    \label{fig:More_Vis_5_SVCT_seed2060_f800_800}
\end{figure}

\begin{figure}[!htbp]
    \centering
    \includegraphics[width=0.88\columnwidth]{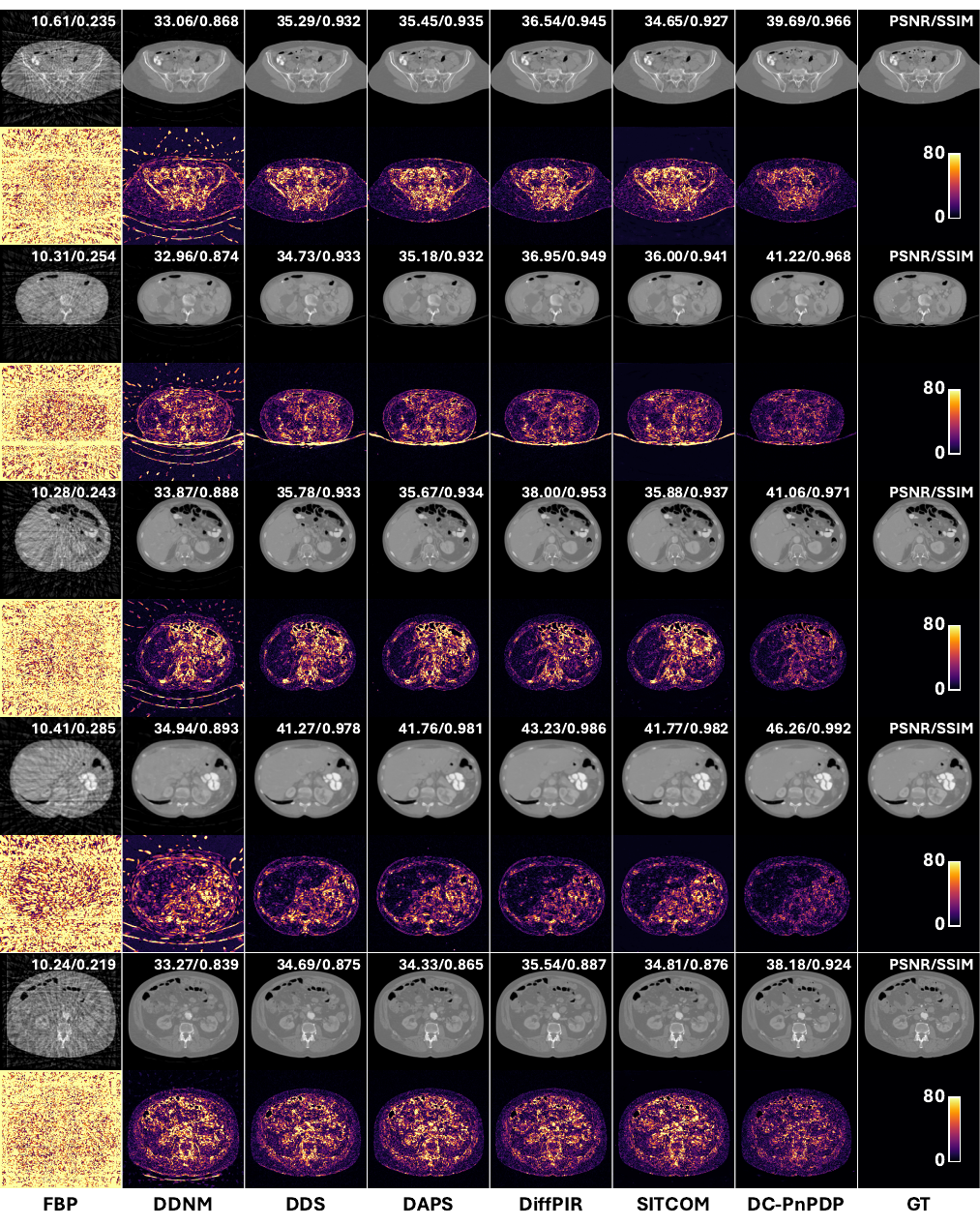}
    \caption{Additional qualitative results on sparse-view CT (SVCT) reconstruction with 20 views. The visualization window is set to $[-800,800]$ HU.}
    \label{fig:More_Vis_5_SVCT_seed2040_f150_250}
\end{figure}

\end{document}